%% file: iclr2027_conference.tex
\documentclass{article} 
\usepackage{iclr2027_conference,times}

\input{math_commands.tex}

\usepackage{hyperref}
\usepackage{amsthm}
\newtheorem{corollary}{Corollary}
\usepackage{url}
\usepackage{amsthm}
\usepackage{subcaption}
\usepackage{graphicx}
\usepackage{booktabs}
\usepackage{amssymb}
\usepackage{algorithm}
\usepackage{algpseudocode}
\usepackage{float}

\newtheorem{theorem}{Theorem}
\newtheorem{lemma}{Lemma}

\title{Orthogonal Witness Control for Muon Optimization
via Sigmoid Spectral Reshaping}

\author{}

\begin{document}

\maketitle

\vspace{-5em}
\begin{center}
{\bf Dat Phi Van$^{1,*,\dagger}$} \quad
{\bf Ngo Vu Minh$^{1,*}$} \quad
{\bf Tuc Nguyen$^{2,*}$} \quad
{\bf Thin Nguyen$^{3}$} \quad
{\bf Ngoc-Thanh Dinh$^{1}$} \quad
{\bf Trung Le$^{4}$}

\vspace{0.5em}
$^{1}$Center for AI Research, VinUniversity, Hanoi, Vietnam \\
$^{2}$Indiana University, Bloomington, IN, USA \\
$^{3}$Deakin University, Burwood, VIC, Australia \\
$^{4}$Monash University, Clayton, VIC, Australia
\end{center}

\let\thefootnote\relax\footnotetext{$^{*}$Equal contribution.}
\let\thefootnote\relax\footnotetext{$^{\dagger}$Corresponding author: Dat Phi Van, \texttt{dat.pv@vinuni.edu.vn}.}

\vspace{1em}

\begin{abstract}
Matrix-valued optimizers such as Muon exploit the spectral structure of neural network updates through Newton--Schulz orthogonalization, but their near-flattening of the singular spectrum discards relative magnitude information across gradient modes. We introduce \emph{Soren} (\textbf{S}pectral \textbf{O}rthogonal \textbf{Re}shapi\textbf{n}g), a matrix-valued optimizer that preserves the singular subspaces of the gradient while applying a bounded, monotone sigmoid transformation to its singular values. This smoothly compresses dominant modes without fully flattening the spectrum. We interpret Soren as a positive-definite preconditioned gradient method and establish convergence guarantees under relative smoothness and metric Polyak--{\L}ojasiewicz geometry. To avoid explicit singular value decomposition, we further develop a finite-depth Soft Newton--Schulz (SNS) polynomial realization of the sigmoid spectral map and characterize how its spectral approximation affects the induced convergence geometry. Experiments across LLM pre-training, supervised fine-tuning, and direct preference optimization demonstrate the effectiveness and robustness of Soren against established optimizers.
\end{abstract}

\section{Introduction}

Optimization methods for deep neural networks increasingly exploit the
matrix structure of neural network parameters rather than treating all
parameters as unstructured vectors. Muon~\citep{jordan2024muon} has recently emerged
as a matrix-valued optimizer for hidden layers. It first constructs a
momentum-based update following SGD~\citep{sgd} with momentum~\citep{momentum},
and then applies Newton--Schulz iterations to approximately orthogonalize
the update. This preserves the singular subspaces of the update while
flattening its active singular spectrum.

This spectral transformation can mitigate the dominance of a small number
of singular modes and produce more balanced update directions. However,
flattening the spectrum also discards relative magnitude information across
gradient modes, pushing weak and strong active directions toward a similar
scale. This motivates a softer form of spectral transformation that can
control spectral imbalance while retaining information about the relative
strengths of different modes.

We introduce \emph{Soren}
(\textbf{S}pectral \textbf{O}rthogonal \textbf{Re}shapi\textbf{n}g), a
sigmoid-based matrix-valued optimizer that preserves the singular
subspaces of the gradient while smoothly reshaping their magnitudes. The
bounded and monotone sigmoid transformation compresses dominant singular
modes while preserving their ordering and the non-uniformity of the
spectrum. In contrast to SGD-type updates, which retain the original
spectral imbalance, and Muon, which strongly flattens the active spectrum,
Soren provides controlled spectral shaping while retaining relative
spectral structure.

We study Soren from both optimization and computational perspectives. The
sigmoid spectral transformation admits a positive-definite preconditioning
interpretation, enabling convergence analysis under relative smoothness and
metric Polyak--{\L}ojasiewicz geometry. To avoid the explicit SVD required
by the exact spectral map, we further develop a finite-depth
Soft Newton--Schulz (SNS) polynomial realization. We also establish a
convergence guarantee for the resulting finite-step SNS update. More
specifically, our contributions are as follows:
\begin{itemize}
\item We introduce \emph{Soren}, a sigmoid-based spectral reshaping
optimizer that preserves the singular subspaces of matrix-valued
gradients while smoothly compressing dominant singular modes without
fully flattening the spectrum.

\item We establish convergence guarantees for Soren under relative
smoothness and metric Polyak--{\L}ojasiewicz geometry. By interpreting
the sigmoid spectral transformation as a positive-definite
preconditioner, our analysis yields a product-form convergence bound
for iteration-dependent metrics and a linear convergence rate under
uniform metric bounds.

\item We develop a finite-depth Soft Newton--Schulz polynomial
realization of the sigmoid spectral map that avoids explicit SVD
computation, and establish a corresponding convergence guarantee for
the finite-step SNS update.

\item We evaluate Soren across LLM pre-training, supervised fine-tuning,
and direct preference optimization, demonstrating consistent
improvements over established optimizers across diverse training
settings.

\end{itemize}

\section{Related work}
\label{related}

\paragraph{Classical and adaptive optimization.}
First-order optimization remains the foundation of modern neural network training. Stochastic gradient descent (SGD) provides the canonical baseline for large-scale optimization~\citep{sgd}, with momentum methods accelerating gradient descent by accumulating past gradients~\citep{momentum}, while Nesterov acceleration introduces a look-ahead correction that improves the behavior of momentum-based updates~\citep{nag}. In deep learning, the choice of momentum and initialization has been shown to strongly affect trainability and convergence~\citep{init_importance}. Adaptive optimizers further improve robustness by rescaling updates using moment estimates, with Adam~\citep{kingma2015adam} and AdamW~\citep{loshchilov2019decoupled} becoming standard choices for large-scale training and fine-tuning. These methods, however, mainly treat parameters as unstructured vectors, leaving matrix structure in neural network layers largely unused.

\paragraph{Matrix-valued spectral optimization.}
Recent work has increasingly exploited the matrix structure of neural network
parameters for optimization. Muon~\citep{jordan2024muon} applies Newton--Schulz
iterations to momentum-based updates of matrix-valued parameters, effectively
approximating their polar factors and flattening the nonzero singular values.
HTMuon~\citep{htmuon} extends this framework by reshaping the singular
spectrum through a heavy-tailed transformation. More recently,
~\citep{qi2026delving} provide a systematic spectral analysis of
Muon and investigate a broader family of spectral transformations. Related
to polar decomposition and Newton--Schulz orthogonalization
~\citep{bjork1971iterative}, Soren instead applies a smooth sigmoid-based
spectral reshaping that compresses dominant singular modes while retaining
the relative ordering and non-uniformity of the spectrum.

\paragraph{Theoretical foundations for spectral preconditioning.}
The spectral interpretation of Muon connects naturally to
curvature-aware and preconditioned optimization. Local quadratic
models and Gauss--Newton approximations provide standard frameworks
for analyzing optimization dynamics near a current iterate
~\citep{martens2010deep}. Relative smoothness
~\citep{bauschke2017relativecanonical,lu2018relativecanonical}
and Polyak--{\L}ojasiewicz (PL) geometry
~\citep{polyak1963gradient,karimi2016ecmlpkddcanonical}
provide a complementary framework for establishing convergence beyond
the standard strongly convex setting. Building on these tools, we
interpret Soren's sigmoid spectral reshaping as a controlled
positive-definite preconditioning of the gradient and analyze both its
exact and finite-step polynomial realizations under the corresponding
metric conditions.

\section{Preliminaries}
\label{Preliminaries}

\subsection{Muon Optimizer}
At iteration $t$, given the current weight $\boldsymbol{W}_{t-1}$, let the full gradient be
$\boldsymbol{G}_t=\nabla\mathcal{L}(\boldsymbol{W}_{t-1})$. Following the spectral discussion in the introduction, we first present Muon in its clean orthogonal-gradient form:
\begin{align}
\boldsymbol{O}_{t} & = \mathrm{Newton\text{-}Schulz}\left(\boldsymbol{G}_{t}\right),\nonumber \\
\boldsymbol{W}_{t} & = \boldsymbol{W}_{t-1}-\eta\boldsymbol{O}_t .
\label{eq:MuonIdeal}
\end{align}
Here $\eta$ denotes the learning rate. The Newton--Schulz iteration approximates the polar factor of its input~\citep{jordan2024muon}. In particular, if $\boldsymbol{G}_{t}=U\Sigma V^\top$ is the singular value decomposition of $\boldsymbol{G}_t$, then the corresponding ideal polar update is
$\left(\boldsymbol{G}_{t}\boldsymbol{G}_{t}^{\top}\right)^{-1/2}\boldsymbol{G}_{t}=UV^\top$.
Therefore, the Muon step preserves the singular subspaces of the gradient while replacing its singular spectrum by a flat one.

The expression above is the form most convenient for analysis. In practical implementations, however, Muon is commonly applied to a momentum-based update rather than the raw gradient itself.

\subsection{Momentum-Induced Gradient Updates}
Momentum-based gradient descent~\citep{momentum} accumulates past gradients through the state
\begin{align}
\boldsymbol{M}_{t} & = \mu\boldsymbol{M}_{t-1}+(1-\mu)\boldsymbol{G}_t,\nonumber \\
\boldsymbol{W}_{t} & = \boldsymbol{W}_{t-1}-\eta\boldsymbol{M}_{t},
\label{eq:MGD}
\end{align}
where $\eta$ is the learning rate and $\mu$ is the momentum coefficient. The momentum buffer is updated using linear interpolation. The state $\boldsymbol{M}_t$ smooths stochastic fluctuations while retaining information from previous iterations.

Nesterov accelerated gradient (NAG)~\citep{nag} introduces a look-ahead correction to the momentum direction:
\begin{align}
\boldsymbol{N}_{t} = (1-\mu)\boldsymbol{G}_t+\mu\boldsymbol{M}_{t}.
\label{eq:NAG}
\end{align}
Practical Muon implementations often apply Newton--Schulz orthogonalization to this Nesterov-style update~\citep{jordan2024muon}:
\begin{align}
\boldsymbol{O}_{t} & = \mathrm{Newton\text{-}Schulz}\left(\boldsymbol{N}_{t}\right),\nonumber \\
\boldsymbol{W}_{t} & = \boldsymbol{W}_{t-1}-\eta\boldsymbol{O}_t .
\label{eq:NAGMuon}
\end{align}
It combines the current gradient with the accumulated momentum before the final step is formed.
Thus, $\boldsymbol{G}_t$, $\boldsymbol{M}_t$, and $\boldsymbol{N}_t$ denote different update quantities, while Muon refers to the spectral transformation applied before updating the weights.

\section{Theoretical Analysis of Sigmoid Spectral Reshaping}
\label{Theoretical}

\subsection{A Geometric Perspective on Spectral Updates}

Let
\(
G_t
=
\nabla\mathcal{L}(W_{t-1})
=
U\Sigma V^\top
=
\sum_i \sigma_i u_i v_i^\top
\)
be the singular value decomposition of the gradient, and define
\(A_i=u_i v_i^\top\). These rank-one matrices are orthonormal under
the Frobenius inner product:
\(
\langle A_i,A_j\rangle_F
=
\operatorname{Tr}(A_i^\top A_j)
=
\delta_{ij},
\)
where \(\delta_{ij}\) denotes the Kronecker delta. Thus, the gradient
admits an orthogonal decomposition into spectral directions \(A_i\),
with \(\sigma_i\) specifying the magnitude of the gradient along each
direction. For an update \(O_t\) that preserves these spectral directions, we can
similarly write
\(
O_t=\sum_i\alpha_i A_i,
\)
where \(\alpha_i\) denotes the magnitude of the update along the
\(i\)-th spectral direction. With
\(W_t=W_{t-1}-\eta O_t\), the resulting parameter displacement is
\[
\Delta W_t
=
W_t-W_{t-1}
=
-\eta O_t
=
-\eta\sum_i\alpha_i A_i.
\]
Hence, from a spectral perspective, an optimizer can be viewed as
transforming the singular-value coefficients \(\{\sigma_i\}\) of the
gradient into the update coefficients \(\{\alpha_i\}\), while
preserving the underlying spectral directions \(\{A_i\}\). Under a first-order Taylor approximation around \(W_{t-1}\), the corresponding loss change is:
\[
\mathcal{L}(W_t)-\mathcal{L}(W_{t-1})
\approx
\left\langle G_t,\Delta W_t\right\rangle_F
=
-\eta\left\langle G_t,O_t\right\rangle_F.
\]
Substituting the spectral decompositions of \(G_t\) and \(O_t\), and
using the orthonormality of the spectral directions \(\{A_i\}\), we obtain:
\[
\left\langle G_t,O_t\right\rangle_F
=
\left\langle\sum_i\sigma_iA_i,\sum_j\alpha_jA_j\right\rangle_F
=
\sum_i\sum_j\sigma_i\alpha_j\left\langle A_i,A_j\right\rangle_F
=
\sum_i\sigma_i\alpha_i.
\]
Hence, the first-order loss change is governed by the interaction between the gradient coefficients $\{\sigma_i\}$ and the update coefficients $\{\alpha_i\}$, with the contribution of each spectral mode to the local descent determined by the product $\sigma_i\alpha_i$.

This spectral view provides a common interpretation of SGD and Muon. SGD sets $\alpha_i=\sigma_i$, giving $\langle G_t,O_t\rangle_F=\sum_i\sigma_i^2$, so dominant modes contribute quadratically to first-order descent. Muon sets $\alpha_i=1$ on active modes, giving $\langle G_t,O_t\rangle_F=\sum_i\sigma_i$, thereby flattening the relative magnitudes across active modes. In contrast, Soren uses the spectral transformation
\(\alpha_i=s(\sigma_i)\), where
\(
s(x)=\frac{1}{1+e^{-x}}
\) is the sigmoid function, giving:
\[
O_t=\sum_i s(\sigma_i)A_i,
\qquad
\left\langle G_t,O_t\right\rangle_F
=
\sum_i\sigma_i s(\sigma_i).
\]
Because \(s(\cdot)\) is strictly increasing and bounded, Soren preserves
the ordering of spectral modes while compressing their relative
magnitudes and bounding the update coefficients. We next establish
convergence guarantees for this spectral update.

\subsection{Convergence under Relative Smoothness and PL Geometry}
\label{convergence}

The spectral reshaping operation in Soren admits an equivalent
positive-definite preconditioning view, which allows us to analyze
convergence under relative smoothness and metric
Polyak--\L{}ojasiewicz (PL) geometry. Let $G_t\in\mathbb{R}^{m\times n}$, $m\le n$, have full row rank at every analyzed step, with thin SVD $G_t=U_t\Sigma_tV_t^\top$ and $\Sigma_t=\operatorname{diag}(\sigma_{t,1},\ldots,\sigma_{t,m})\succ0$. The spectral map $s(G_t)=U_ts(\Sigma_t)V_t^\top$ induces the following positive-definite preconditioner:
\begin{equation}
P_t = U_ts(\Sigma_t)\Sigma_t^{-1}U_t^\top = U_t\operatorname{diag}\!\left(\frac{s(\sigma_{t,i})}{\sigma_{t,i}}\right)U_t^\top \succ0.
\label{eq:sigmoid_preconditioner}
\end{equation}
Thus, $s(G_t)=P_tG_t$. Writing $w_t=\operatorname{vec}(W_t)$, $g_t=\operatorname{vec}(G_t)$, where $\operatorname{vec}(\cdot)$
denotes column-wise vectorization, and defining $\mathcal{P}_t=I_n\otimes P_t$, the update can be written exactly as:
\begin{equation}
w_t=w_{t-1}-\eta_t\mathcal{P}_tg_t.
\label{eq:sigmoid_vector_update}
\end{equation}
The derivation is given in Appendix~\ref{subsec:app_preconditioner}. For any positive-definite matrix $M$, define $\|v\|_M^2=v^\top Mv$. We measure parameter displacements in the $\mathcal{P}_t^{-1}$ norm and gradients in the dual $\mathcal{P}_t$ norm. The metric is determined by the gradient at $w_{t-1}$ and held fixed
throughout the one-step analysis. Let $\Omega_t$ be a convex region containing the update segment. Assume that $\mathcal{L}$ is twice continuously differentiable on a neighborhood of $\Omega_t$, with full Hessian $H(z)=\nabla^2\mathcal{L}(z)$ and finite optimal value $\mathcal{L}^*$.

With $\mathcal{P}_t$ fixed, define $\widetilde{\beta}_t$ as the smallest relative-smoothness constant:
\begin{equation}
\widetilde{\beta}_t = \inf\left\{\beta>0: H(z)\preceq\beta\mathcal{P}_t^{-1} \ \text{for all }z\in\Omega_t\right\},
\label{eq:beta_def}
\end{equation}
and define $\widetilde{\alpha}_t$ as the largest metric Polyak--\L{}ojasiewicz (PL) constant:
\begin{equation}
\widetilde{\alpha}_t = \inf_{\substack{w\in\Omega_t\\ \mathcal{L}(w)>\mathcal{L}^*}} \frac{\tfrac12\|\nabla\mathcal{L}(w)\|_{\mathcal{P}_t}^2} {\mathcal{L}(w)-\mathcal{L}^*}.
\label{eq:alpha_def}
\end{equation}
Equation~\ref{eq:beta_def} characterizes relative smoothness with respect to the quadratic reference induced by $\mathcal P_t^{-1}$ ~\citep{bauschke2017relativecanonical,lu2018relativecanonical}, while Equation~\ref{eq:alpha_def} defines the corresponding $\mathcal P_t$-metric PL constant ~\citep{karimi2016ecmlpkddcanonical}. Together, they determine the per-step contraction factor $1-\widetilde{\alpha}_t/\widetilde{\beta}_t$. Under the step size specified below, the descent lemma together with the lower bound $\mathcal{L}\ge\mathcal{L}^*$ ensures that $\widetilde{\alpha}_t\le\widetilde{\beta}_t$. The metric formulations and descent lemma are derived in Appendix~\ref{subsec:app_metric_geometry}, with the convergence proof given in Appendix~\ref{subsec:app_exact_convergence}.

\begin{theorem}[Soren convergence]
\label{thm:exact_soren_convergence}
Under the assumptions above, suppose that the line segment between
$w_{t-1}$ and $w_t$ lies in $\Omega_t$ for every iteration $t$, and that
$0<\widetilde{\alpha}_t<\infty$ and
$0<\widetilde{\beta}_t<\infty$. With the step size
$\eta_t=1/\widetilde{\beta}_t$, the exact Soren update in
Eq.~\ref{eq:sigmoid_vector_update} satisfies, for any integer
$T\ge1$,
\begin{equation}
\mathcal{L}(w_T)-\mathcal{L}^*
\le
\left[
\prod_{t=1}^{T}
\left(1-\frac{\widetilde{\alpha}_t}{\widetilde{\beta}_t}\right)
\right]
\left(\mathcal{L}(w_0)-\mathcal{L}^*\right).
\label{eq:exact_product_rate}
\end{equation}
Here, $t$ indexes the optimization iterations and $T$ denotes the
number of iterations under consideration. If there exist constants
$0<\overline{\alpha}\le\overline{\beta}<\infty$ such that
\begin{equation}
\widetilde{\alpha}_t\ge\overline{\alpha},
\qquad
\widetilde{\beta}_t\le\overline{\beta},
\qquad
\forall t,
\label{eq:uniform_metric_bounds}
\end{equation}
then
\begin{equation}
\mathcal{L}(w_T)-\mathcal{L}^*
\le
\left(1-\frac{\overline{\alpha}}{\overline{\beta}}\right)^T
\left(\mathcal{L}(w_0)-\mathcal{L}^*\right).
\label{eq:exact_uniform_rate}
\end{equation}
\end{theorem}

The product bound allows the metric to vary across iterations, whereas
the uniform bounds yield a linear convergence rate. The proof in
Appendix~\ref{subsec:app_exact_convergence} combines the full-Hessian
descent lemma with the metric PL condition.

\subsection{Polynomial Realization via Soft Newton--Schulz Iteration}
\label{polynomial}
\label{subsec:sns_exact_comparison}

To avoid an explicit SVD, we use a finite-depth \emph{Soft Newton--Schulz (SNS) iteration}, inspired by Newton--Schulz orthogonalization~\citep{bjork1971iterative}. The identity $\operatorname{sigmoid}(x)=\frac12(1+\tanh(x/2))$ motivates two polynomial streams. Initialize $Q_0=G_t/(\|G_t\|_F+\epsilon)$ and $T_0=G_t/4$, and apply
\begin{equation}
\Phi(X)=\frac12(3I-XX^\top)X.
\label{eq:sns_recurrence}
\end{equation}
Applying $K$ iterations to $Q$ and two iterations to $T$, we define
\begin{equation}
\widehat{s}(G_t)=\frac12(Q_K+T_2).
\label{eq:sns_output}
\end{equation}
Algorithms~\ref{alg:sns}--\ref{alg:soren} in Appendix~\ref{app:algorithms} provide pseudocode, while Figure~\ref{fig:sns_comparison} compares SNS with the sigmoid map.
Both streams remain diagonal in the singular-vector basis of $G_t=U_t\Sigma_tV_t^\top$, yielding coefficients $\widehat{s}_{t,i}$. To compare SNS with the exact map at the same gradient iterate, assume
\begin{equation}
|\widehat{s}_{t,i}-s(\sigma_{t,i})|
\le \rho_t s(\sigma_{t,i}),\qquad 0\le\rho_t<1,\qquad \forall i.
\label{eq:relative_spectral_error}
\end{equation}
Under the full-row-rank assumption of Section~\ref{convergence}, this ensures a positive-definite SNS preconditioner. Let $\widehat{\alpha}_t$ and $\widehat{\beta}_t$ denote the corresponding
SNS metric PL and relative-smoothness constants, respectively, defined
as in equation~\ref{eq:beta_def}--\ref{eq:alpha_def} using the SNS metric
on the same region $\Omega_t$ and with the same $\mathcal{L}^*$. Appendix~\ref{subsec:app_sns_metric_details} gives the explicit definitions and preconditioner bounds.

\begin{theorem}[Finite-step geometry preservation]
\label{thm:sns_geometry_preservation}
Under equation~\ref{eq:relative_spectral_error} and $0<\widetilde{\alpha}_t\le\widetilde{\beta}_t<\infty$, the corresponding exact and SNS constants satisfy
\begin{equation}
\frac{1-\rho_t}{1+\rho_t} \frac{\widetilde{\alpha}_t}{\widetilde{\beta}_t} \le \frac{\widehat{\alpha}_t}{\widehat{\beta}_t} \le \frac{1+\rho_t}{1-\rho_t} \frac{\widetilde{\alpha}_t}{\widetilde{\beta}_t}.
\label{eq:condition_ratio_sandwich}
\end{equation}
\end{theorem}

Thus, the guaranteed per-step progress retains at least a fraction $(1-\rho_t)/(1+\rho_t)$ of that for the exact map at the same iterate. The individual metric bounds and proof appear in Appendices~\ref{subsec:app_sns_metric_details}--\ref{subsec:app_geometry_preservation}. These bounds directly yield the following convergence guarantee for the
finite-step SNS update.
\begin{corollary}[SNS convergence]
\label{cor:sns_convergence}
Under the assumptions of Section~\ref{convergence}, suppose each SNS update segment lies in $\Omega_t$ and, along the SNS trajectory, Eq.~(\plaineqref{eq:relative_spectral_error}) holds with $\rho_t\le\overline{\rho}<1$, $\widetilde{\alpha}_t\ge\overline{\alpha}>0$, and $\widetilde{\beta}_t\le\overline{\beta}<\infty$. Then, with $\eta_t=1/\widehat{\beta}_t$,
\begin{equation}
\mathcal{L}(w_T)-\mathcal{L}^* \le \left[ 1- \frac{1-\overline{\rho}}{1+\overline{\rho}} \frac{\overline{\alpha}}{\overline{\beta}} \right]^T \left(\mathcal{L}(w_0)-\mathcal{L}^*\right).
\label{eq:sns_uniform_rate}
\end{equation}
\end{corollary}

The resulting factor $(1-\overline{\rho})/(1+\overline{\rho})$ quantifies the degradation in the guaranteed progress induced by the finite-step approximation. This convergence guarantee is conditional on the spectral-error and uniform metric bounds; the proof is given in Appendix~\ref{subsec:app_sns_convergence}.

\section{Experiments}
\label{experiments}

We evaluate Soren across pre-training and post-training, including supervised fine-tuning (SFT)~\citep{sft} and direct preference optimization (DPO)~\citep{dpo}. We compare against optimizers that are closely related to our work, including AdamW~\citep{loshchilov2019adamw}, Muon~\citep{jordan2024muon}, and HTMuon~\citep{htmuon}.

\subsection{LLM Pre-training}

\paragraph{Experimental Setup.}

We pretrain various LLaMA-family models (LLaMA-60M, LLaMA-135M, LLaMA-350M,
LLaMA-1B~\citep{touvron2023llama}) on
the C4 dataset~\citep{raffel2020exploring}. 

\begin{table}[htbp]
    \centering
    \small
    \setlength{\tabcolsep}{12pt}
    
    \caption{Comparison of established pretraining optimizers across LLaMA models of varying sizes on the C4 dataset. Lower perplexity indicates better performance. Detailed hyperparameter configurations are reported in Tables~\ref{tab:llama_hyperparameters}–\ref{tab:llama_350m_1b_hyperparameters} in Appendix~\ref{hyperparam}.
    }
    \label{tab:optimizer_results}
    
    \begin{tabular}{lcccc}
        \toprule
        & \textbf{LLaMA-60M} 
        & \textbf{LLaMA-135M}
        & \textbf{LLaMA-350M}
        & \textbf{LLaMA-1B} \\
        \midrule
        
        AdamW   & 44.54 & 36.96 & 23.49 & 20.82  \\
        Muon  & 36.86  & 27.42 & 21.73 & 18.41\\
        HTMuon   & 35.32 & \textbf{26.15} & 22.58 & 18.44  \\
        
        \textbf{Soren}
        & \textbf{35.14}
        & \textbf{26.15}
        & \textbf{21.53}
        & \textbf{18.29} \\
        
        \bottomrule
    \end{tabular}
\end{table}

Table~\ref{tab:optimizer_results} compares Soren with established pre-training optimizers across four LLaMA model scales on the C4 dataset. Soren achieves the lowest perplexity on three of the four model scales and ties for the best result on the remaining scale, demonstrating consistently strong performance across model sizes. Compared with Muon, Soren reduces perplexity by $1.72$, $1.57$, $0.20$, and $0.12$ on LLaMA-60M, LLaMA-135M, LLaMA-350M, and LLaMA-1B, respectively. Relative to AdamW, Soren achieves reductions of $9.40$, $10.81$, $1.96$, and $2.53$, respectively. Compared with HTMuon, Soren achieves lower perplexity on LLaMA-60M, LLaMA-350M, and LLaMA-1B by $0.18$, $1.05$, and $0.15$, respectively, while matching HTMuon on LLaMA-135M at a perplexity of $26.15$. Overall, Soren consistently outperforms AdamW and Muon across all model scales while achieving comparable or better performance than HTMuon, indicating that Soren provides a competitive optimization approach for LLaMA pre-training.

\begin{figure}[htbp]
    \centering

    \includegraphics[width=0.49\linewidth]{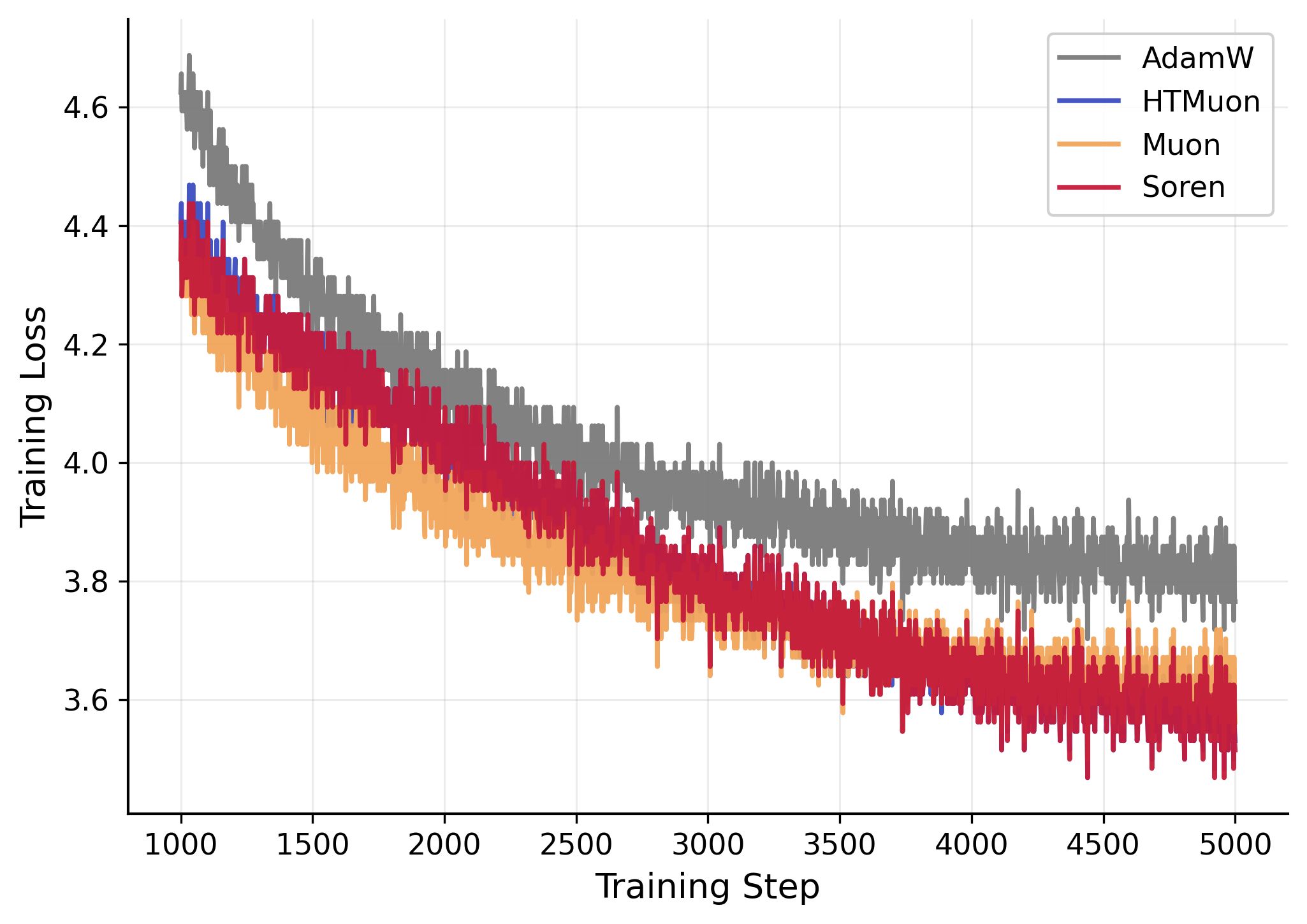}
    \hfill
    \includegraphics[width=0.49\linewidth]{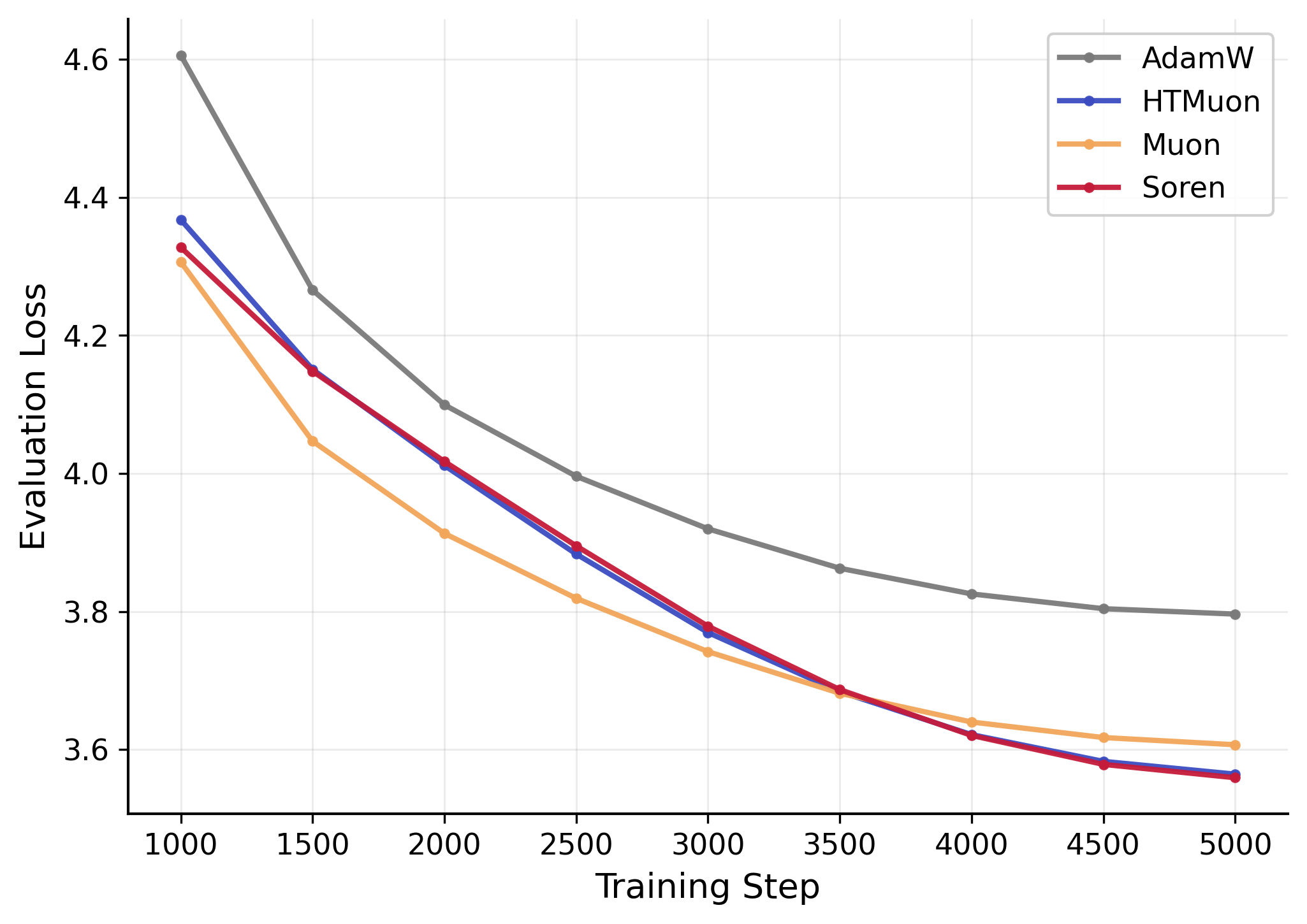}

    \vspace{0.5em}

    \includegraphics[width=0.49\linewidth]{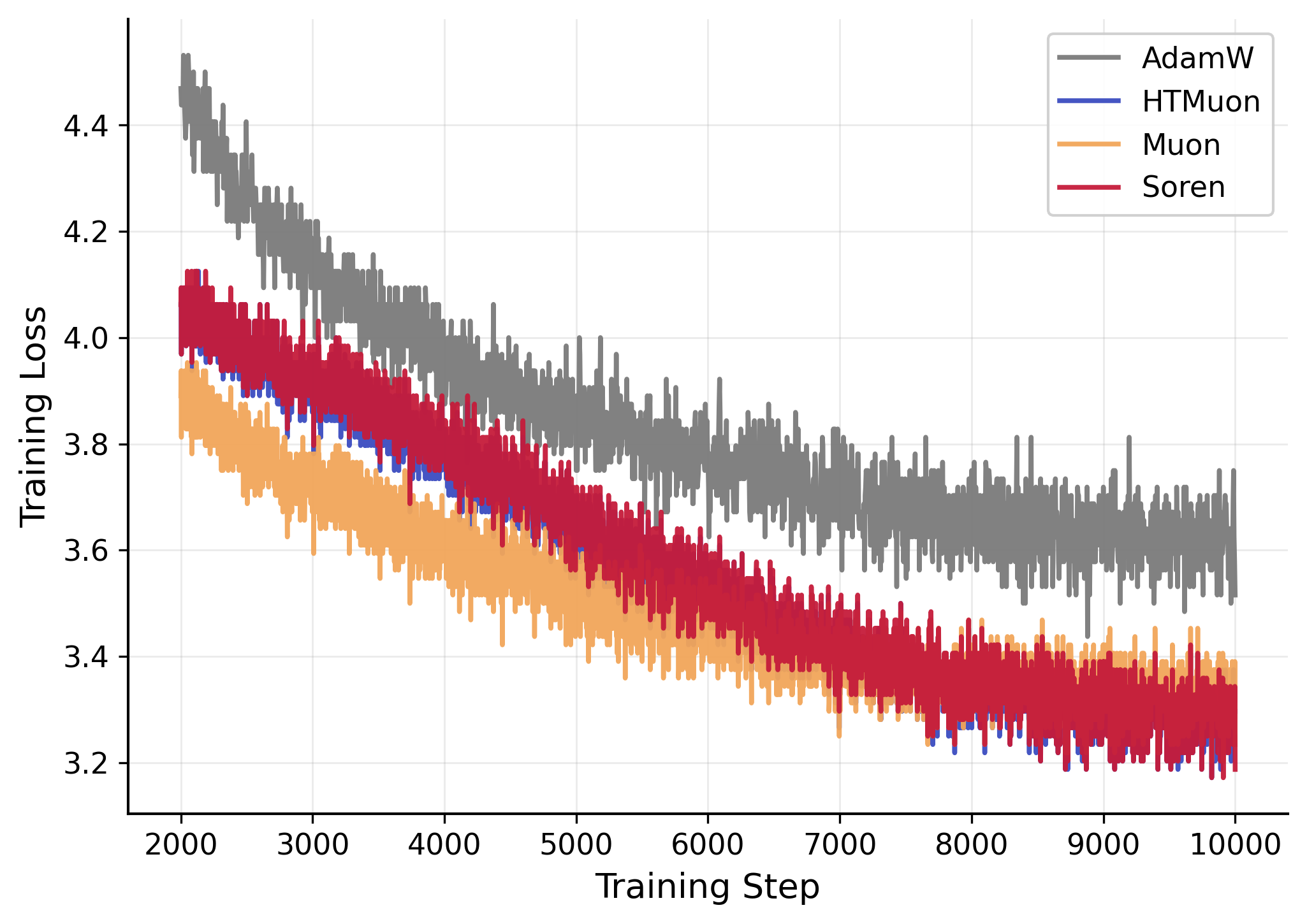}
    \hfill
    \includegraphics[width=0.49\linewidth]{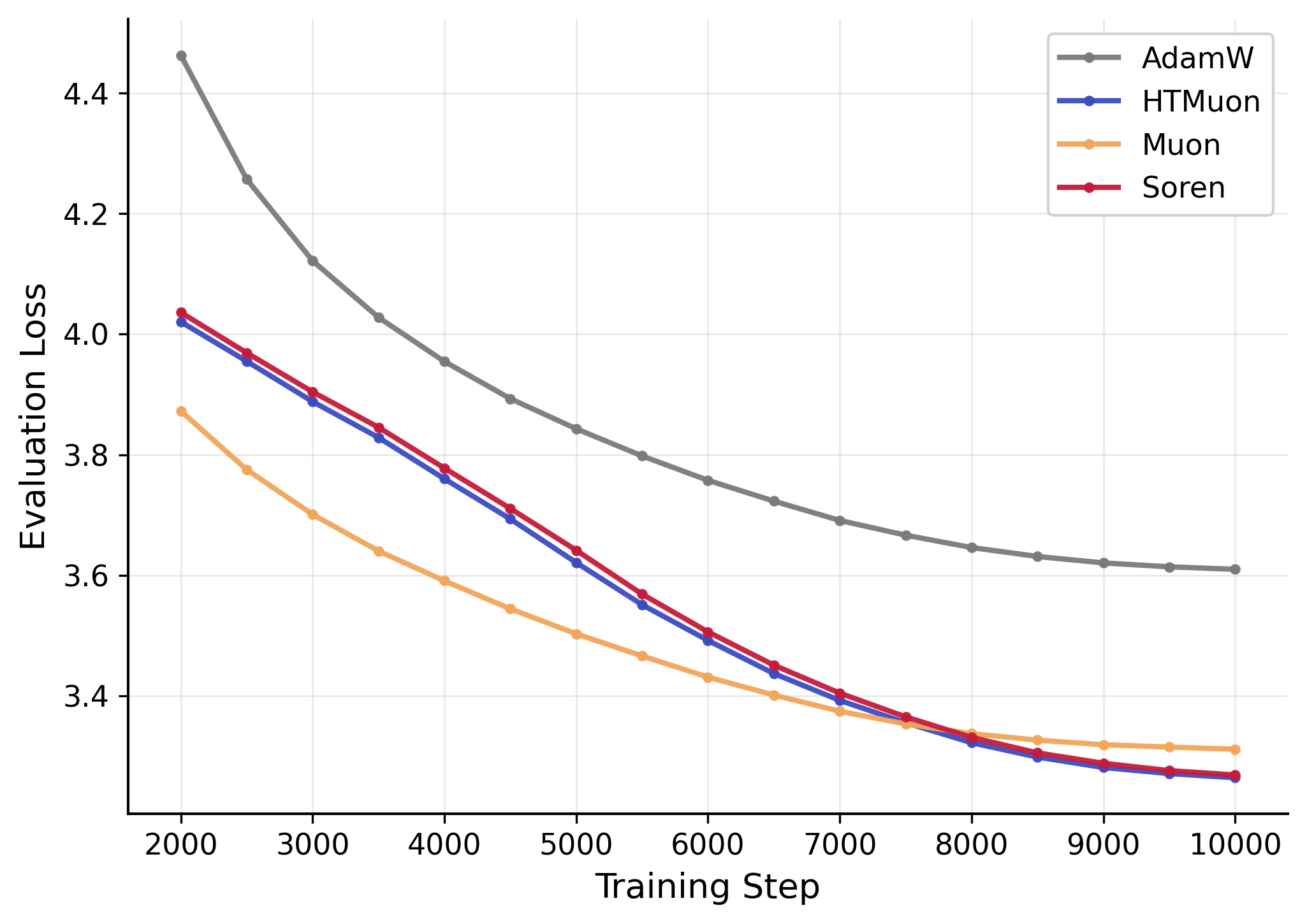}

    \caption{
        Training and evaluation loss curves for the LLaMA-60M (top row) and LLaMA-135M (bottom row) experiments.
    }
    \label{fig:loss}
\end{figure}

In figure~\ref{fig:loss}, we present the training and evaluation loss curves for LLaMA-60M and LLaMA-135M on the C4 dataset. Soren exhibits convergence behavior closely aligned with HTMuon, particularly in the later stages of training, with comparable training and evaluation loss trajectories across both model scales. Compared with Muon, Soren achieves similar or lower losses toward the end of training, while consistently outperforming AdamW in the later stages. Notably, Soren reaches evaluation losses comparable to or slightly lower than HTMuon toward the end of training, consistent with the final perplexity results in Table~\ref{tab:optimizer_results}. In figure~\ref{fig:loss1b} in appendix~\ref{more_expe}, we provide the training loss
curves for LLaMA-350M and LLaMA-1B. These results indicate that Soren exhibits stable convergence behavior and competitive optimization performance during pre-training.

\paragraph{Downstream Tasks}

To further assess the applicability of Soren to broader NLP tasks, we evaluate a Qwen3-0.6B~\citep{yang2025qwen3} model pre-trained with Soren on six commonsense reasoning benchmarks using the lm-eval-harness framework. For all tasks, we use the default prompts and conduct evaluation in a zero-shot setting.

\begin{table}[t]
    \centering
    \setlength{\tabcolsep}{4.5pt}
    \caption{Zero-shot evaluation results ($\uparrow$) on six commonsense reasoning benchmarks for the Qwen3-0.6B model pretrained with different optimizers.}
    \label{tab:commonsense_results}
    \begin{tabular}{lccccccc}
        \toprule
        Optimizer & ARC-c & ARC-e & PIQA & HellaSwag & OBQA & WinoGrande  & Avg. \\
        \midrule
        AdamW
        & 23.13 & 25.05 & 50.16  & 25.49 & 24.00 & 48.80 & 32.77 \\
        Muon
        & 23.46 & 26.89 & 50.11 & 26.97 & 24.40 & 50.51 & 33.72 \\
        HTMuon
        & \textbf{25.77}  & 26.81 & 50.71 & 27.25 & 24.60 & 49.88 & 34.17 \\
        Soren
        & 25.51 & \textbf{27.36} & \textbf{51.25}
        & \textbf{28.47} & \textbf{25.80} & \textbf{51.99} & \textbf{35.06} \\
        \bottomrule
    \end{tabular}
\end{table}

As shown in Table~\ref{tab:commonsense_results}, Soren achieves the best
performance on five out of six commonsense reasoning benchmarks and obtains
the highest average score of $35.06$ across all tasks. Compared with HTMuon,
which achieves the second-highest average score of $34.17$, Soren improves the
overall performance by $0.89$ points. Soren also consistently outperforms
AdamW and Muon across all six benchmarks, while achieving particularly strong
gains on ARC-e, PIQA, HellaSwag, OBQA, and WinoGrande. These results highlight
the effectiveness of Soren for pre-training language models and its ability to
improve performance across a diverse set of commonsense reasoning tasks.

\subsection{LLM Supervised Fine-tuning}

\paragraph{Experimental Setup.}
We conduct supervised fine-tuning experiments on Qwen2.5-3B~\citep{qwen2.5}
using the UltraChat 200k dataset~\citep{ding-etal-2023-enhancing}.
We use the train\_sft split for training and follow the standard data format
provided by the dataset. To assess the general capabilities of the resulting models,
we follow the Hugging Face Open LLM Leaderboard evaluation protocol~\citep{open-llm-leaderboard}
and use the Language Model Evaluation Harness~\citep{eval-harness}.
We evaluate performance on six benchmarks spanning commonsense reasoning,
multi-task understanding, factuality, and mathematical reasoning:
ARC~\citep{clark2018think}, HellaSwag~\citep{zellers2019hellaswag},
Winogrande~\citep{sakaguchi2020winogrande}, MMLU~\citep{hendrycks2021measuring},
TruthfulQA~\citep{lin-etal-2022-truthfulqa}, and GSM8K~\citep{cobbe2021training}.

\begin{table}[htbp]
    \centering
    \small
    \caption{Evaluation results of different optimizers for SFT on Qwen2.5-3B. Detailed hyperparameter configurations are reported in Table~\ref{tab:sft_dpo_hyperparameters} in Appendix~\ref{hyperparam}.}
    \label{tab:qwen25_3b_sft}

    \resizebox{\linewidth}{!}{%
    \begin{tabular}{lccccccc}
        \toprule
        \textbf{Method}
        & \textbf{HellaSwag}
        & \textbf{ARC}
        & \textbf{MMLU}
        & \textbf{TruthfulQA}
        & \textbf{Winogrande}
        & \textbf{GSM8K}
        & \textbf{Average} \\
        \midrule

        AdamW
        & 73.08$_{\pm
        0.44}$
        & 55.97$_{\pm 1.45}$
        & 61.79$_{\pm 0.39}$
        & 46.75$_{\pm 1.48}$
        & 68.90$_{\pm 1.30}$
        & 64.44$_{\pm 1.32}$
        & 61.82 \\

        Muon
        & 73.72$_{\pm 0.44}$
        & 56.14$_{\pm 1.45}$
        & 62.83$_{\pm 0.39}$
        & 47.44$_{\pm 1.49}$
        & 68.90$_{\pm 1.30}$
        & \underline{69.83$_{\pm 1.26}$}
        & 63.14 \\

        HTMuon
        & \underline{73.75$_{\pm 0.44}$}
        & \underline{56.31$_{\pm 1.45}$}
        & \underline{62.93$_{\pm 0.39}$}
        & \underline{47.50$_{\pm 1.48}$}
        & \underline{70.24$_{\pm 1.30}$}
        & 69.22$_{\pm 1.27}$
        & \underline{63.33} \\

        \textbf{Soren (Ours)}
        & \textbf{74.31$_{\pm 0.44}$}
        & \textbf{57.17$_{\pm 1.45}$}
        & \textbf{65.62$_{\pm 0.38}$}
        & \textbf{48.14$_{\pm 1.50}$}
        & \textbf{71.11$_{\pm 1.27}$}
        & \textbf{72.55$_{\pm 1.23}$}
        & \textbf{64.82} \\

        \bottomrule
    \end{tabular}%
    }
\end{table}

Soren achieves the highest score on all six evaluation benchmarks as well as the highest average score of 64.82, outperforming AdamW, Muon, and HTMuon by 3.00, 1.68, and 1.49 points, respectively. Compared with HTMuon, Soren improves performance by 0.56 points on HellaSwag, 0.86 points on ARC, 2.69 points on MMLU, 0.64 points on TruthfulQA, 0.87 points on Winogrande, and 2.72 points on GSM8K. Notably, the gains are particularly pronounced on MMLU and GSM8K, demonstrating consistent improvements across diverse capabilities, including commonsense reasoning, knowledge-intensive understanding, factuality, and mathematical reasoning. The sigmoid spectral reshaping in Soren smoothly compresses dominant singular modes while preserving relative differences among weaker modes, which can produce more balanced gradient updates during fine-tuning and contribute to the observed performance improvements.

\subsection{LLM Direct Preference Optimization}

\paragraph{Experimental Setup.}
We conduct direct preference optimization experiments starting from the
HuggingFaceH4/mistral-7b-sft-alpha checkpoint, a supervised
fine-tuned version of Mistral-7B-v0.1~\citep{jiang2023mistral}.
We use this SFT checkpoint as the initial policy for subsequent preference
optimization stage. Specifically, we perform DPO on the
UltraFeedback Binarized dataset~\citep{pmlr-v235-cui24f},
which provides preference pairs consisting of chosen and rejected responses.
For evaluation, we follow the same evaluation protocol and benchmark suite
used in the SFT experiments described above.

\begin{table}[htbp]
    \centering
    \small
    \caption{Evaluation results of different optimizers for DPO on Mistral 7B. Detailed hyperparameter configurations are reported in Table~\ref{tab:sft_dpo_hyperparameters} in Appendix~\ref{hyperparam}.}
    \label{tab:mistral_dpo}

    \resizebox{\linewidth}{!}{%
    \begin{tabular}{lccccccc}
        \toprule
        \textbf{Method}
        & \textbf{HellaSwag}
        & \textbf{ARC}
        & \textbf{MMLU}
        & \textbf{TruthfulQA}
        & \textbf{Winogrande}
        & \textbf{GSM8K}
        & \textbf{Average} \\
        \midrule

        AdamW
        & 80.69$_{\pm 0.36}$
        & 61.09$_{\pm 1.42}$
        & 51.09$_{\pm 0.39}$
        & 45.32$_{\pm 1.58}$
        & 76.40$_{\pm 1.14}$
        & 30.93$_{\pm 1.27}$
        & 57.59 \\

        Muon
        & \underline{81.89$_{\pm 0.38}$}
        & \underline{61.95$_{\pm 1.42}$}
        & 55.07$_{\pm 0.40}$
        & \textbf{49.56$_{\pm 1.57}$}
        & \underline{77.35$_{\pm 1.18}$}
        & \underline{31.69$_{\pm 1.28}$}
        & \underline{59.59} \\

        HTMuon
        & 81.85$_{\pm 0.38}$
        & 61.86$_{\pm 1.42}$
        & \underline{55.13$_{\pm 0.40}$}
        & 49.52$_{\pm 1.57}$
        & 76.95$_{\pm 1.18}$
        & 31.31$_{\pm 1.28}$
        & 59.44 \\

        \textbf{Soren (Ours)}
        & \textbf{83.25$_{\pm 0.37}$}
        & \textbf{62.46$_{\pm 1.42}$}
        & \textbf{57.21$_{\pm 0.39}$}
        & \underline{49.54$_{\pm 1.56}$}
        & \textbf{77.74$_{\pm 1.17}$}
        & \textbf{36.32$_{\pm 1.32}$}
        & \textbf{61.09} \\

        \bottomrule
    \end{tabular}%
    }
\end{table}

On Mistral-7B, Soren achieves the highest average performance of 61.09, outperforming AdamW, Muon, and HTMuon by 3.50, 1.50, and 1.65 points, respectively. Soren also achieves the best performance on five of the six benchmarks, including HellaSwag (83.25), ARC (62.46), MMLU (57.21), Winogrande (77.74), and GSM8K (36.32), while remaining highly competitive on TruthfulQA (49.54), where it is only marginally below Muon (49.56). Compared with HTMuon, Soren yields improvements of 1.40 points on HellaSwag, 0.60 points on ARC, 2.08 points on MMLU, 0.79 points on Winogrande, and 5.01 points on GSM8K. The particularly large gains on MMLU and GSM8K demonstrate the effectiveness of Soren for preference optimization across knowledge-intensive and mathematical reasoning tasks.

\subsection{Singular Spectrum Analysis}
\label{sec:spectrum}

Figure~\ref{fig:spectrum_layer} visualizes the singular spectra of representative gradient matrices from the first and final layers of the LLaMA-60M model at different stages of training. As shown in the figure, the original gradient exhibits a pronounced spectral decay, with a large gap between dominant and weaker singular modes, causing the update to be disproportionately influenced by a small number of high-magnitude directions while suppressing the relative contribution of weaker modes. In contrast, Muon largely flattens the spectrum by assigning nearly equal magnitude to the active singular directions, removing much of the relative spectral structure of the original gradient. Soren exhibits an intermediate behavior: it substantially compresses the gap between large and small singular values while preserving their ordering and non-uniformity. Thus, dominant modes remain the primary contributors to the update, but their relative dominance is reduced, allowing weaker modes to contribute more substantially without being forced to the same magnitude. This spectral behavior, consistently observed across layers and training stages, highlights Soren's ability to mitigate excessive spectral imbalance while retaining informative relative differences among singular directions.

\begin{figure}[t]
    \centering

    \begin{subfigure}[b]{0.32\textwidth}
        \centering
        \includegraphics[width=\textwidth]{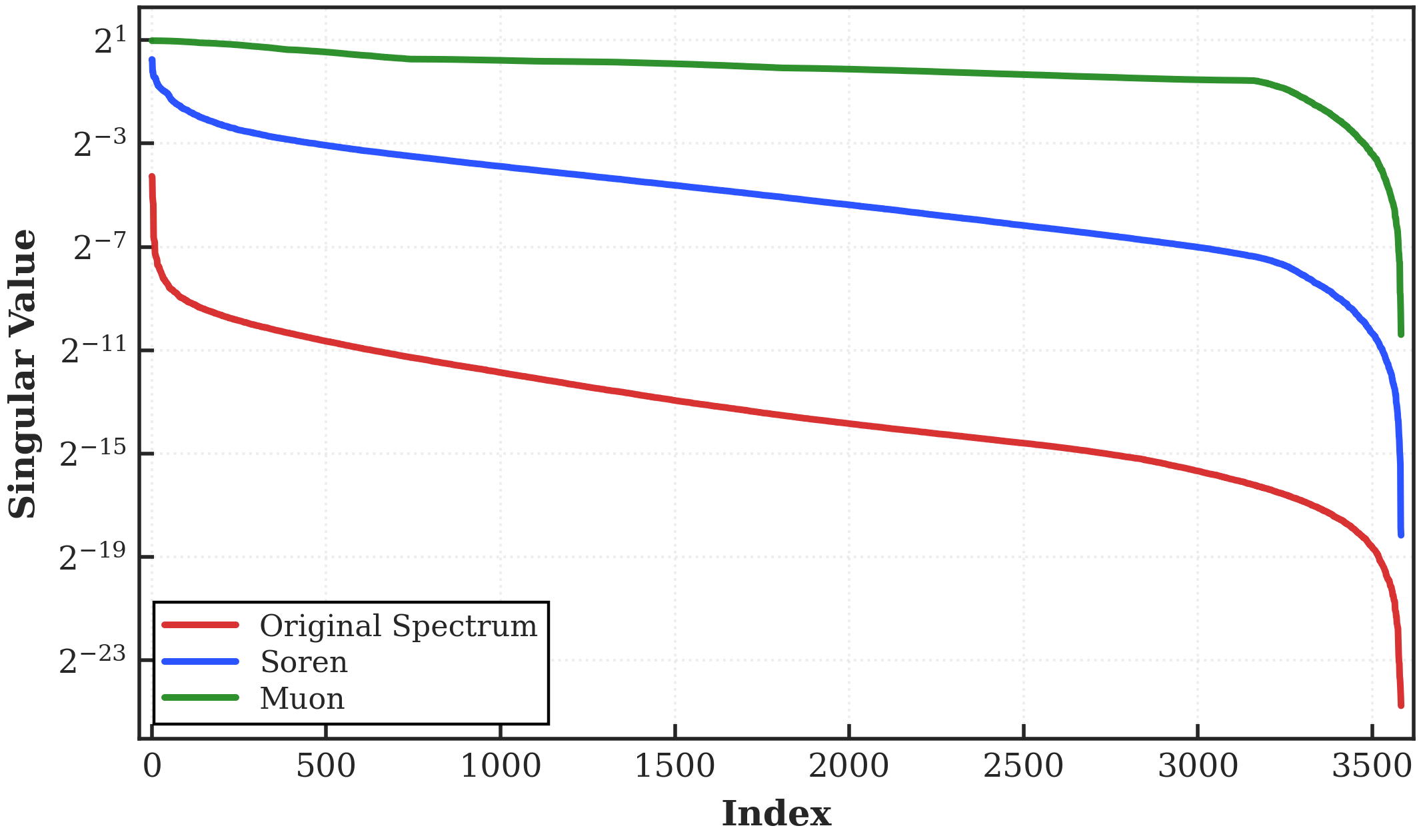}
        \caption{Step 1}
        \label{fig:spectrum_layer0_step1}
    \end{subfigure}
    \hfill
    \begin{subfigure}[b]{0.32\textwidth}
        \centering
        \includegraphics[width=\textwidth]{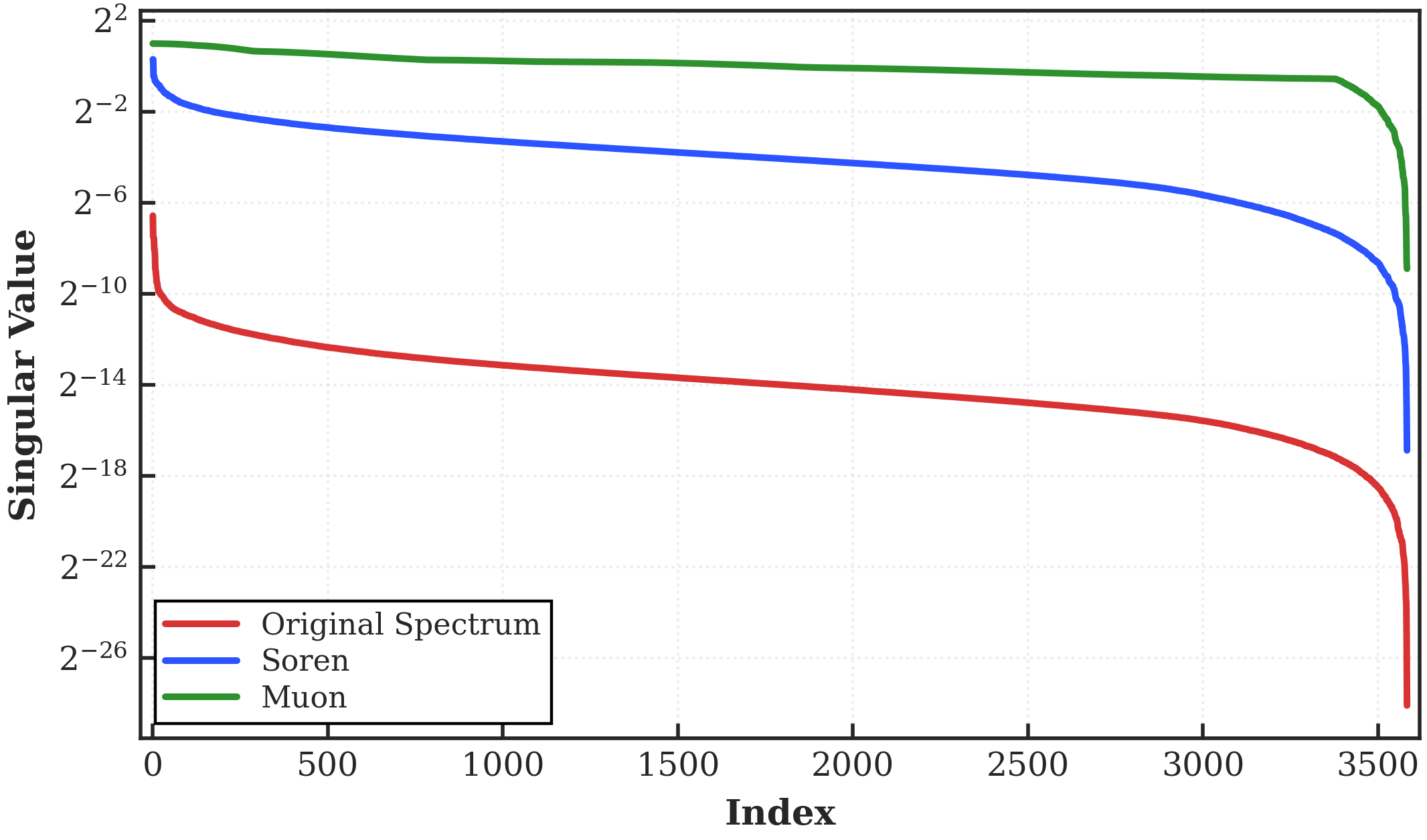}
        \caption{Step 2500}
        \label{fig:spectrum_layer0_step2500}
    \end{subfigure}
    \hfill
    \begin{subfigure}[b]{0.32\textwidth}
        \centering
        \includegraphics[width=\textwidth]{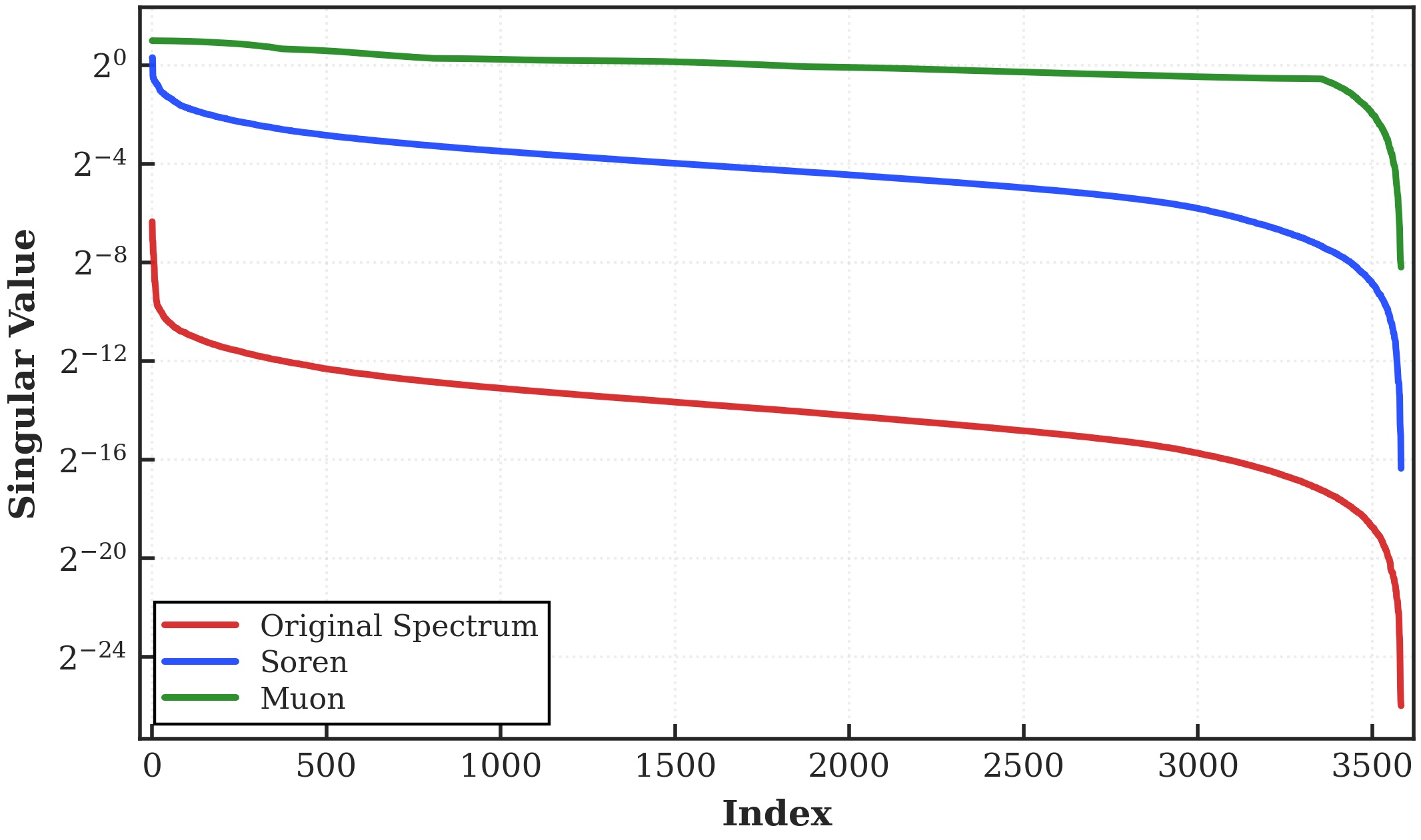}
        \caption{Step 5000}
        \label{fig:spectrum_layer0_step5000}
    \end{subfigure}

    \vspace{4pt}

    \begin{subfigure}[b]{0.32\textwidth}
        \centering
        \includegraphics[width=\textwidth]{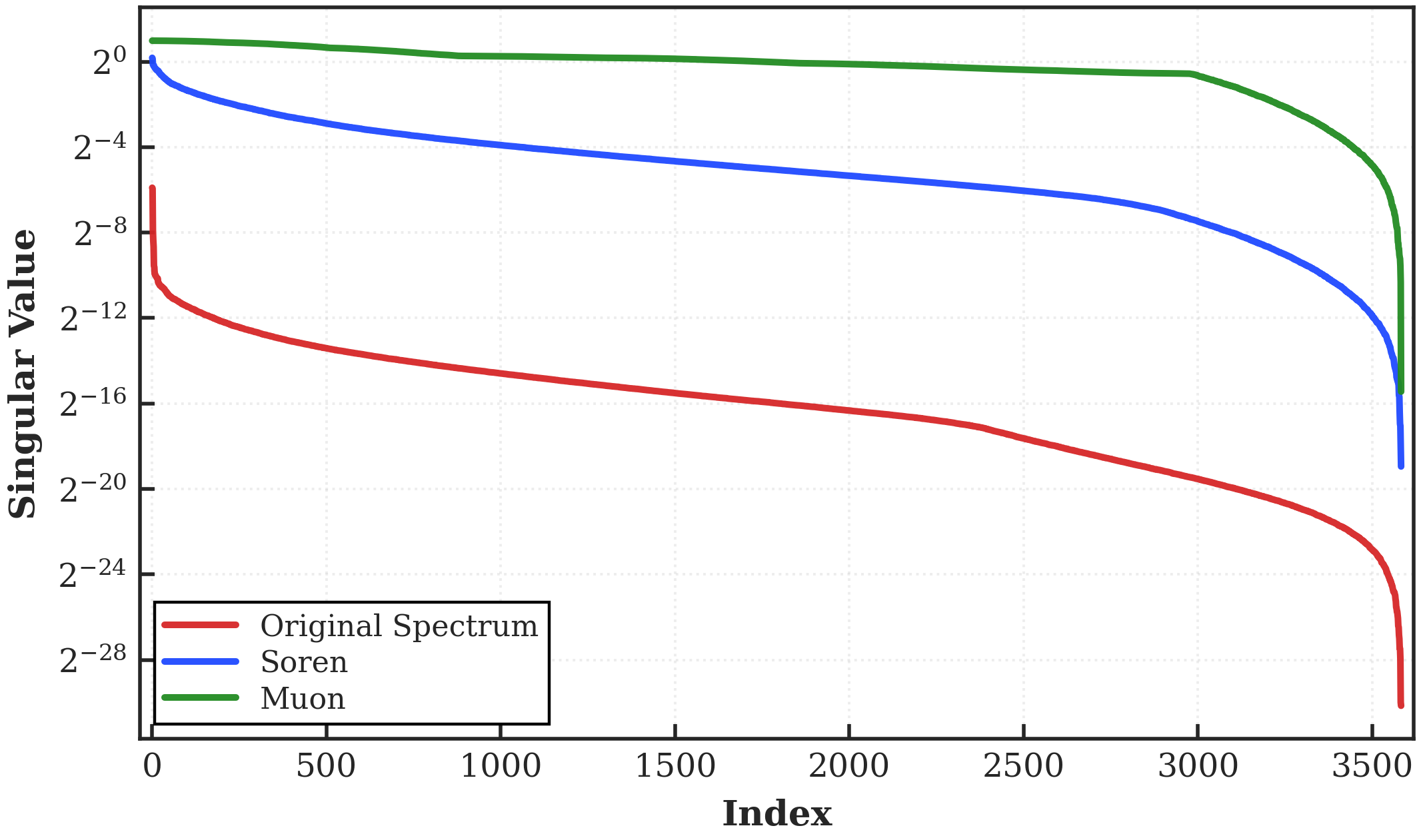}
        \caption{Step 1}
        \label{fig:spectrum_layer7_step1}
    \end{subfigure}
    \hfill
    \begin{subfigure}[b]{0.32\textwidth}
        \centering
        \includegraphics[width=\textwidth]{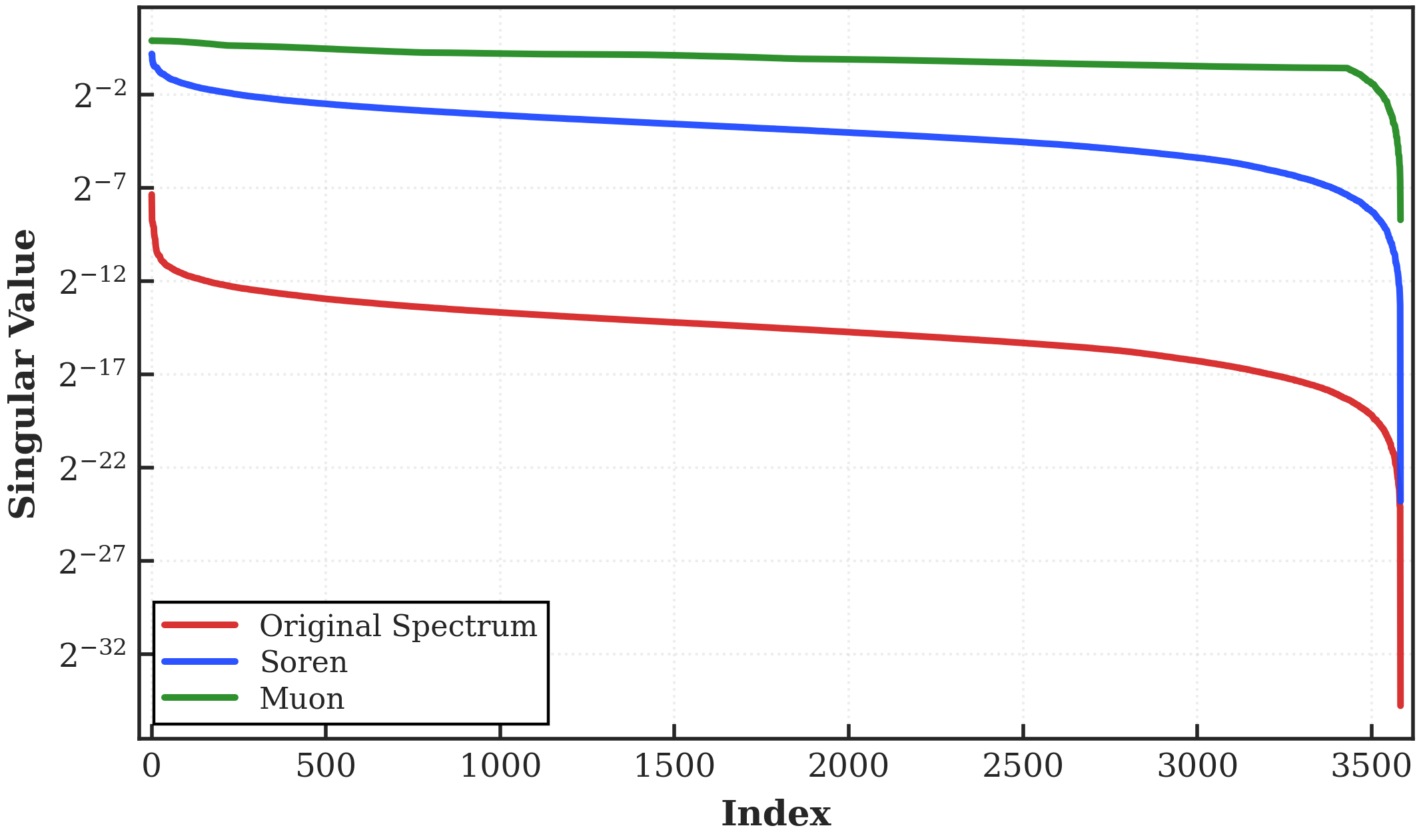}
        \caption{Step 2500}
        \label{fig:spectrum_layer7_step2500}
    \end{subfigure}
    \hfill
    \begin{subfigure}[b]{0.32\textwidth}
        \centering
        \includegraphics[width=\textwidth]{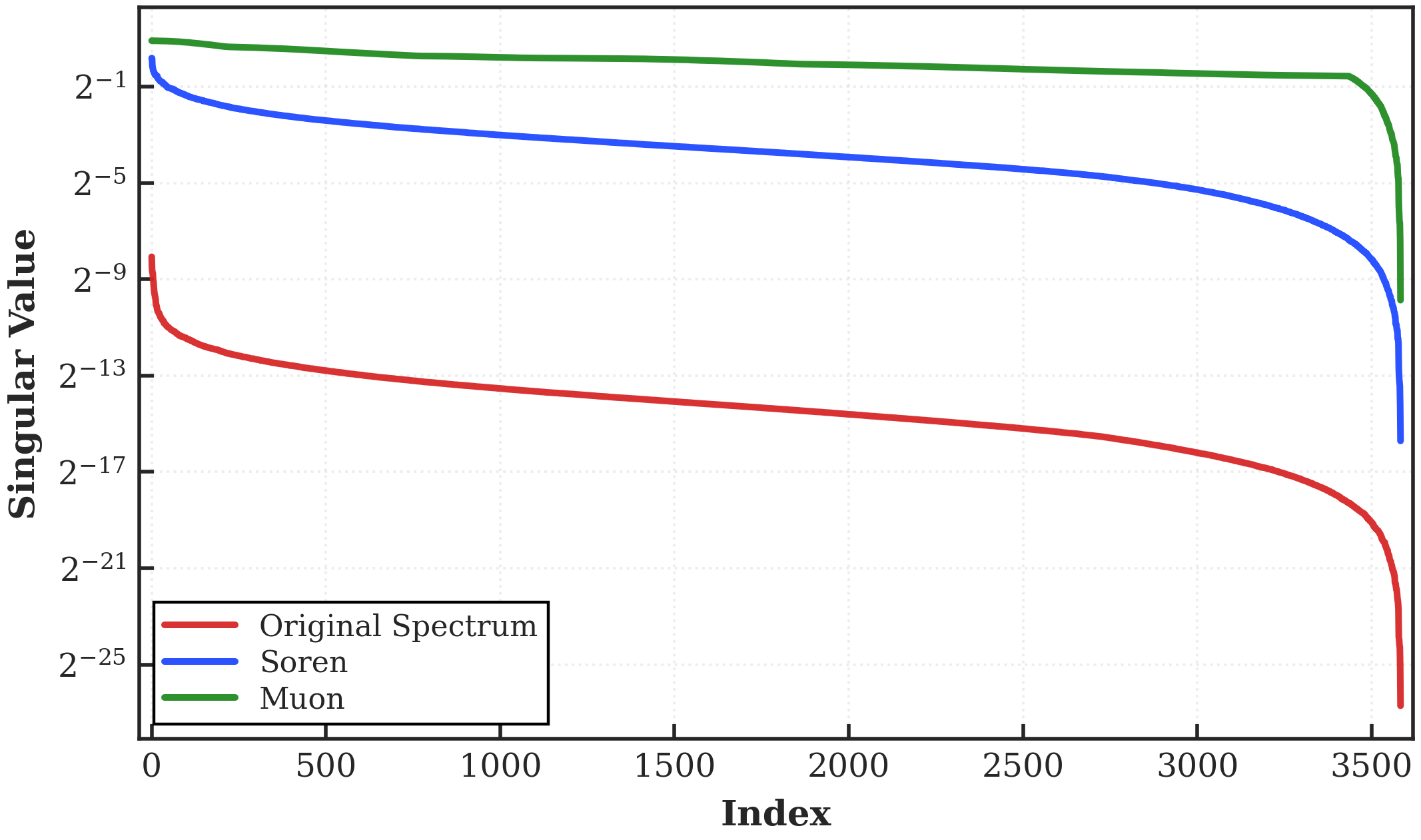}
        \caption{Step 5000}
        \label{fig:spectrum_layer7_step5000}
    \end{subfigure}

    \caption{Spectral transformation of a representative gradient matrix from the Llama 60M model, with singular values sorted in descending order. The top and bottom rows correspond to the first and final layers, respectively, while the columns show the initial, intermediate, and final training steps. Soren smoothly compresses dominant modes while preserving weaker-mode structure, in contrast to Muon-style orthogonalization, which flattens active modes toward a nearly uniform scale.}
    \label{fig:spectrum_layer}
\end{figure}

\subsection{Soft Newton--Schulz Iteration Analysis}
\label{subsec:sns_iteration_analysis}

We empirically evaluate how closely the finite-step Soft Newton--Schulz (SNS) iteration approximates the exact sigmoid spectral map using the maximum relative error (MaxRelErr). Specifically, for singular values $x\in\mathcal{X}\subset[0.01,1]$, we define
$s(x)=\operatorname{sigmoid}(x)$ and denote the corresponding SNS approximation by $\widehat{s}(x)$. We measure the maximum pointwise relative error as:
\[
\mathrm{MaxRelErr}
=
\max_{x\in\mathcal{X}}
\frac{|\widehat{s}(x)-s(x)|}{s(x)}.
\]

\begin{figure}[htbp]
    \centering

    \begin{subfigure}[t]{0.49\linewidth}
        \centering
        \includegraphics[width=\linewidth]{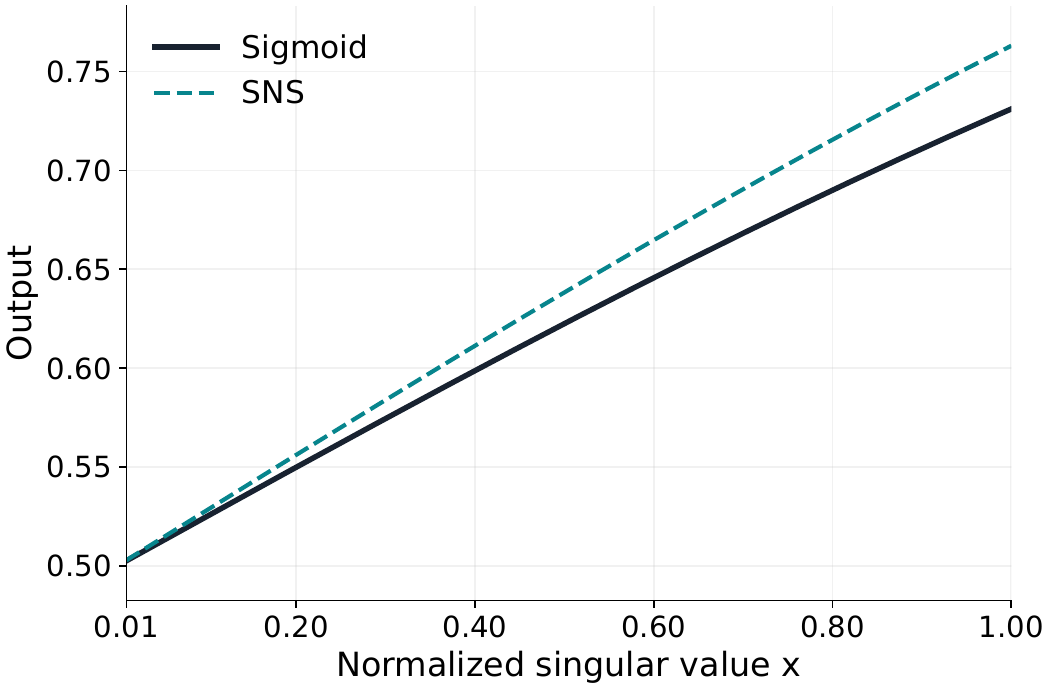}
        \caption{SNS and sigmoid on normalized singular values.}
        \label{fig:sns_scalar_comparison}
    \end{subfigure}
    \hfill
    \begin{subfigure}[t]{0.49\linewidth}
        \centering
        \includegraphics[width=\linewidth]{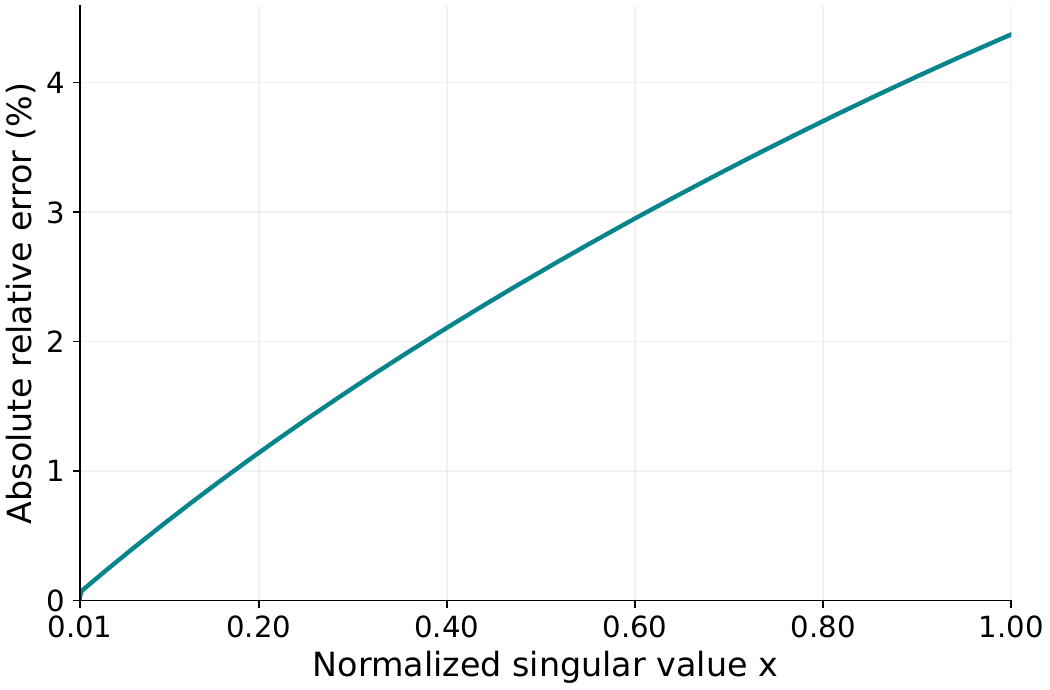}
        \caption{Absolute relative error $|\widehat{s}(x)-s(x)|/s(x)$.}
        \label{fig:sns_scalar_relative_error}
    \end{subfigure}

    \caption{Comparison between SNS and the sigmoid approximation.}
    \label{fig:sns_comparison}
\end{figure}

Figure~\ref{fig:sns_scalar_comparison} compares the two scalar mappings, while Figure~\ref{fig:sns_scalar_relative_error} shows their pointwise relative error. Over the evaluated range, the SNS approximation closely follows the exact sigmoid mapping, with a maximum relative error of $4.37\%$. This metric has the same form as the relative spectral error introduced in Eq.~(\plaineqref{eq:relative_spectral_error}), providing an empirical assessment of the approximation quality at the scalar level. Beyond this scalar-level analysis, we provide spectral comparisons on LLM training in Appendix~\ref{app:soren_exact_spectra} and evaluate mean and maximum absolute errors on CNN gradients in Appendix~\ref{app:sns_metrics}.

\section{Conclusion}
\label{sec:conclusion}

We introduced \emph{Soren}, a spectral reshaping optimizer that applies a
sigmoid transformation to the singular values of gradient updates, providing
a smooth alternative to both magnitude-preserving and spectral
orthogonalization updates. We established convergence guarantees under
relative smoothness and Polyak--\L{}ojasiewicz geometry, and developed a
finite-step Soft Newton--Schulz realization that avoids explicit SVD
computation. Experiments across language model pre-training, supervised
fine-tuning, and direct preference optimization demonstrate that Soren
achieves competitive or improved performance over established optimizers.
Overall, our results show that controlled spectral reshaping is a promising
approach for optimization in neural language model training and alignment.

\bibliography{iclr2027_conference}
\bibliographystyle{iclr2027_conference}

\appendix

\section{Mathematical proofs and derivations}

\subsection{Exact spectral preconditioner and vectorization}
\label{subsec:app_preconditioner}
This appendix derives the preconditioned form used in Section~\ref{convergence}. Let $G_t\in\mathbb{R}^{m\times n}$ with $m\le n$ have full row rank and thin SVD
\[
G_t=U_t\Sigma_tV_t^\top,
\]
where $U_t\in\mathbb{R}^{m\times m}$ is orthogonal, $V_t\in\mathbb{R}^{n\times m}$ has orthonormal columns, and $\Sigma_t=\operatorname{diag}(\sigma_{t,1},\ldots,\sigma_{t,m})$ with $\sigma_{t,i}>0$. The exact spectral Soren map is
\[
s(G_t)=U_ts(\Sigma_t)V_t^\top.
\]
Using Eq.~(\plaineqref{eq:sigmoid_preconditioner}),
\begin{align}
P_tG_t = U_ts(\Sigma_t)\Sigma_t^{-1}U_t^\top U_t\Sigma_tV_t^\top
= U_ts(\Sigma_t)V_t^\top =s(G_t).
\end{align}
Since $s(\sigma_{t,i})>0$ and $\sigma_{t,i}>0$, every eigenvalue $s(\sigma_{t,i})/\sigma_{t,i}$ of $P_t$ is positive, hence $P_t\succ0$.

Let $w_t=\operatorname{vec}(W_t)$ and $g_t=\operatorname{vec}(G_t)$. The identity
\[
\operatorname{vec}(ABC)=(C^\top\otimes A)\operatorname{vec}(B)
\]
gives
\begin{align}
\operatorname{vec}(P_tG_t) =\operatorname{vec}(P_tG_tI_n)
=(I_n\otimes P_t)\operatorname{vec}(G_t)
=\mathcal{P}_tg_t.
\end{align}
Therefore the matrix update $W_t=W_{t-1}-\eta_tP_tG_t$ is equivalent to Eq.~(\plaineqref{eq:sigmoid_vector_update}).

\subsection{Metric relative smoothness and metric PL geometry}
\label{subsec:app_metric_geometry}
Fix an iteration $t$ and hold $\mathcal{P}_t$ constant. The associated quadratic reference function is
\[
h_t(w)=\frac12w^\top\mathcal{P}_t^{-1}w.
\]
Then $\nabla^2h_t(w)=\mathcal{P}_t^{-1}$. Thus
\[
H(z)\preceq\widetilde{\beta}_t\mathcal{P}_t^{-1}
\]
is precisely the Hessian form of relative smoothness with respect to the quadratic reference $h_t$~\citep{bauschke2017relativecanonical,lu2018relativecanonical}.

The primal metric norm and its dual norm are
\[
\|d\|_{\mathcal{P}_t^{-1}}^2=d^\top\mathcal{P}_t^{-1}d, \qquad \|g\|_{\mathcal{P}_t}^2=g^\top\mathcal{P}_tg.
\]
For fixed $\mathcal{P}_t$, the linear coordinate transformation $w=\mathcal{P}_t^{1/2}y$ gives
\[
\nabla_y\mathcal{L}(\mathcal{P}_t^{1/2}y) =\mathcal{P}_t^{1/2}\nabla_w\mathcal{L}(w),
\]
and hence
\[
\|\nabla_y\mathcal{L}\|_2^2 =\nabla_w\mathcal{L}(w)^\top\mathcal{P}_t\nabla_w\mathcal{L}(w).
\]
The standard PL/gradient-dominance condition~\citep{karimi2016ecmlpkddcanonical}, applied to the transformed objective $y\mapsto\mathcal{L}(\mathcal{P}_t^{1/2}y)$, therefore takes the form of Eq.~(\plaineqref{eq:alpha_def}) in the original coordinates. This change of coordinates explains the metric formulation; positivity of the metric PL constant remains an assumption of our analysis.

We next derive the descent inequality used in Theorem~\ref{thm:exact_soren_convergence}. Let $w,w'=w+d\in\Omega_t$. Since $\Omega_t$ is convex, $w+\tau d\in\Omega_t$ for $\tau\in[0,1]$. The integral form of Taylor's theorem gives
\begin{equation}
\mathcal{L}(w+d) = \mathcal{L}(w) +\nabla\mathcal{L}(w)^\top d +\int_0^1(1-\tau)d^\top H(w+\tau d)d\,d\tau.
\label{eq:app_integral_taylor}
\end{equation}
By Eq.~(\plaineqref{eq:beta_def}),
\[
d^\top H(w+\tau d)d \le \widetilde{\beta}_t d^\top\mathcal{P}_t^{-1}d.
\]
Using $\int_0^1(1-\tau)d\tau=1/2$ in Eq.~(\plaineqref{eq:app_integral_taylor}),
\begin{equation}
\mathcal{L}(w+d) \le \mathcal{L}(w) +\nabla\mathcal{L}(w)^\top d +\frac{\widetilde{\beta}_t}{2} \|d\|_{\mathcal{P}_t^{-1}}^2.
\label{eq:app_metric_descent}
\end{equation}
This inequality follows directly from the Hessian bound along the segment and is valid without a quadratic approximation of the objective.

\subsection{Proof of Theorem~\ref{thm:exact_soren_convergence}}
\label{subsec:app_exact_convergence}
\begin{proof}
Set
\[
d_t=w_t-w_{t-1}=-\eta_t\mathcal{P}_tg_t, \qquad g_t=\nabla\mathcal{L}(w_{t-1}).
\]
Applying Eq.~(\plaineqref{eq:app_metric_descent}) at $w_{t-1}$ gives
\begin{align}
\mathcal{L}(w_t) &\le \mathcal{L}(w_{t-1}) -\eta_tg_t^\top\mathcal{P}_tg_t +\frac{\widetilde{\beta}_t}{2} \eta_t^2 (\mathcal{P}_tg_t)^\top \mathcal{P}_t^{-1} (\mathcal{P}_tg_t)\\
&= \mathcal{L}(w_{t-1}) -\eta_t\left(1-\frac{\eta_t\widetilde{\beta}_t}{2}\right) \|g_t\|_{\mathcal{P}_t}^2.
\label{eq:app_one_step_descent}
\end{align}
With $\eta_t=1/\widetilde{\beta}_t$,
\begin{equation}
\mathcal{L}(w_t) \le \mathcal{L}(w_{t-1}) -\frac{1}{2\widetilde{\beta}_t} \|g_t\|_{\mathcal{P}_t}^2.
\label{eq:app_optimal_metric_step}
\end{equation}
The full-row-rank assumption gives $g_t\ne0$, hence $\mathcal{L}(w_{t-1})>\mathcal{L}^*$, since a differentiable objective has zero gradient at a global minimizer. Combining Eq.~(\plaineqref{eq:app_optimal_metric_step}) with $\mathcal{L}(w_t)\ge\mathcal{L}^*$ and Eq.~(\plaineqref{eq:alpha_def}) therefore yields
\[
\widetilde{\alpha}_t
\le \frac{\tfrac12\|g_t\|_{\mathcal{P}_t}^2}{\mathcal{L}(w_{t-1})-\mathcal{L}^*}
\le \widetilde{\beta}_t.
\]
Thus the contraction factor below lies in $[0,1)$; strict inequality between the two constants is not required.
By Eq.~(\plaineqref{eq:alpha_def}), evaluated at $w_{t-1}$,
\[
\frac12\|g_t\|_{\mathcal{P}_t}^2 \ge \widetilde{\alpha}_t \left(\mathcal{L}(w_{t-1})-\mathcal{L}^*\right).
\]
Substituting this bound into Eq.~(\plaineqref{eq:app_optimal_metric_step}) yields
\[
\mathcal{L}(w_t)-\mathcal{L}^* \le \left(1-\frac{\widetilde{\alpha}_t}{\widetilde{\beta}_t}\right) \left(\mathcal{L}(w_{t-1})-\mathcal{L}^*\right).
\]
Recursively applying the one-step inequality for $t=1,\ldots,T$ proves Eq.~(\plaineqref{eq:exact_product_rate}).

Under Eq.~(\plaineqref{eq:uniform_metric_bounds}),
\[
\frac{\widetilde{\alpha}_t}{\widetilde{\beta}_t} \ge \frac{\overline{\alpha}}{\overline{\beta}},
\]
so every factor in the product satisfies
\[
1-\frac{\widetilde{\alpha}_t}{\widetilde{\beta}_t} \le 1-\frac{\overline{\alpha}}{\overline{\beta}}.
\]
Applying this bound to Eq.~(\plaineqref{eq:exact_product_rate}) proves Eq.~(\plaineqref{eq:exact_uniform_rate}).
\end{proof}

\subsection{SNS metric definitions and comparison bounds}
\label{subsec:app_sns_metric_details}

For the finite-step SNS output in Section~\ref{polynomial}, singular-subspace preservation gives
\begin{equation}
\widehat{s}(G_t)=U_t\widehat{S}_tV_t^\top,
\qquad \widehat{S}_t=\operatorname{diag}(\widehat{s}_{t,1},\ldots,\widehat{s}_{t,m}).
\label{eq:sns_spectral_representation}
\end{equation}
Under the full-row-rank assumption of Section~\ref{convergence}, define
\begin{equation}
\widehat{P}_t=U_t\operatorname{diag}\!\left(\frac{\widehat{s}_{t,i}}{\sigma_{t,i}}\right)U_t^\top,
\qquad \widehat{\mathcal{P}}_t=I_n\otimes\widehat{P}_t.
\label{eq:sns_preconditioner}
\end{equation}
Then $\widehat{s}(G_t)=\widehat{P}_tG_t$. The relative error bound in Eq.~(\plaineqref{eq:relative_spectral_error}) makes all coefficients positive and hence both preconditioners positive definite.

\begin{lemma}[Preconditioner sandwich]
\label{lemma:preconditioner_sandwich}
Under Eq.~(\plaineqref{eq:relative_spectral_error}), the exact and finite-step preconditioners satisfy:
\begin{equation}
(1-\rho_t)\mathcal{P}_t \preceq \widehat{\mathcal{P}}_t \preceq (1+\rho_t)\mathcal{P}_t,
\label{eq:preconditioner_sandwich}
\end{equation}
and, consequently,
\begin{equation}
\frac{1}{1+\rho_t}\mathcal{P}_t^{-1} \preceq \widehat{\mathcal{P}}_t^{-1} \preceq \frac{1}{1-\rho_t}\mathcal{P}_t^{-1}.
\label{eq:inverse_preconditioner_sandwich}
\end{equation}
\end{lemma}

To compare the two metrics under the same local geometry, we use the same region $\Omega_t$ and optimal value $\mathcal{L}^*$ to define the SNS constants:
\begin{align}
\widehat{\beta}_t &= \inf\left\{\beta>0: H(z)\preceq\beta\widehat{\mathcal{P}}_t^{-1} \ \text{for all }z\in\Omega_t\right\},
\label{eq:sns_beta_def}\\
\widehat{\alpha}_t &= \inf_{\substack{w\in\Omega_t\\ \mathcal{L}(w)>\mathcal{L}^*}} \frac{\tfrac12\|\nabla\mathcal{L}(w)\|_{\widehat{\mathcal{P}}_t}^2} {\mathcal{L}(w)-\mathcal{L}^*}.
\label{eq:sns_alpha_def}
\end{align}

The individual metric bounds underlying Theorem~\ref{thm:sns_geometry_preservation} are
\begin{align}
(1-\rho_t)\widetilde{\alpha}_t &\le \widehat{\alpha}_t \le (1+\rho_t)\widetilde{\alpha}_t,
\label{eq:alpha_sandwich}\\
(1-\rho_t)\widetilde{\beta}_t &\le \widehat{\beta}_t \le (1+\rho_t)\widetilde{\beta}_t.
\label{eq:beta_sandwich}
\end{align}

The following subsections establish the spectral representation, preconditioner sandwich, and metric bounds, then prove the theorem and convergence corollary.

\subsection{Singular-subspace preservation of finite-step SNS}
\label{subsec:app_sns_subspace}
Represent a stream iterate in the fixed singular-vector basis of its input as
\[
X_k=UD_kV^\top, \qquad D_k=\operatorname{diag}(d_{k,1},\ldots,d_{k,m}).
\]
The diagonal entries of $D_k$ are spectral coefficients and need not be nonnegative after a polynomial step. Since $U$ is orthogonal and $V$ has orthonormal columns,
\[
X_kX_k^\top=UD_k^2U^\top.
\]
Using Eq.~(\plaineqref{eq:sns_recurrence}),
\[
\Phi(X_k)=\frac12(3I-UD_k^2U^\top)UD_kV^\top =U\left[\frac12(3I-D_k^2)D_k\right]V^\top.
\]
Thus, the recurrence acts only on the diagonal coefficients in the fixed basis $U,V$. Both $Q_0$ and $T_0$ are scalar multiples of $G_t$, so they share the input basis $U_t,V_t$. Induction over any finite number of steps in each stream gives
\[
Q=U_tD_QV_t^\top, \qquad T=U_tD_TV_t^\top.
\]
Consequently,
\[
\frac12(Q+T) = U_t\left[\frac12(D_Q+D_T)\right]V_t^\top,
\]
which proves Eq.~(\plaineqref{eq:sns_spectral_representation}).

\subsection{SNS preconditioner representation and sandwich}
\label{subsec:app_sns_sandwich}
From Eq.~(\plaineqref{eq:sns_preconditioner}),
\[
\widehat{P}_tG_t =U_t\operatorname{diag}\!\left(\frac{\widehat{s}_{t,i}}{\sigma_{t,i}}\right)U_t^\top U_t\Sigma_tV_t^\top =U_t\widehat{S}_tV_t^\top =\widehat{s}(G_t).
\]

Equation~(\plaineqref{eq:relative_spectral_error}) is equivalent to
\[
(1-\rho_t)s(\sigma_{t,i}) \le \widehat{s}_{t,i} \le (1+\rho_t)s(\sigma_{t,i}).
\]
Dividing by $\sigma_{t,i}>0$ and using the common eigenbasis $U_t$ gives
\[
(1-\rho_t)P_t \preceq \widehat{P}_t \preceq (1+\rho_t)P_t.
\]
Taking the Kronecker product with $I_n\succeq0$ preserves Loewner order and proves Eq.~(\plaineqref{eq:preconditioner_sandwich}). Since every matrix is positive definite and inversion reverses Loewner order,
\[
\widehat{\mathcal{P}}_t\preceq(1+\rho_t)\mathcal{P}_t \quad\Longrightarrow\quad \frac{1}{1+\rho_t}\mathcal{P}_t^{-1} \preceq \widehat{\mathcal{P}}_t^{-1},
\]
and
\[
(1-\rho_t)\mathcal{P}_t\preceq\widehat{\mathcal{P}}_t \quad\Longrightarrow\quad \widehat{\mathcal{P}}_t^{-1} \preceq \frac{1}{1-\rho_t}\mathcal{P}_t^{-1}.
\]
Together these inequalities prove Eq.~(\plaineqref{eq:inverse_preconditioner_sandwich}).

\subsection{Preservation of the metric PL constant}
\label{subsec:app_alpha_preservation}
For every $w\in\Omega_t$ with $\mathcal{L}(w)>\mathcal{L}^*$, Eq.~(\plaineqref{eq:preconditioner_sandwich}) implies
\[
(1-\rho_t) \|\nabla\mathcal{L}(w)\|_{\mathcal{P}_t}^2 \le \|\nabla\mathcal{L}(w)\|_{\widehat{\mathcal{P}}_t}^2 \le (1+\rho_t) \|\nabla\mathcal{L}(w)\|_{\mathcal{P}_t}^2.
\]
Dividing by $2(\mathcal{L}(w)-\mathcal{L}^*)>0$ yields
\[
(1-\rho_t)\frac{\tfrac12\|\nabla\mathcal{L}(w)\|_{\mathcal{P}_t}^2}{\mathcal{L}(w)-\mathcal{L}^*} \le \frac{\tfrac12\|\nabla\mathcal{L}(w)\|_{\widehat{\mathcal{P}}_t}^2}{\mathcal{L}(w)-\mathcal{L}^*} \le (1+\rho_t)\frac{\tfrac12\|\nabla\mathcal{L}(w)\|_{\mathcal{P}_t}^2}{\mathcal{L}(w)-\mathcal{L}^*}.
\]
Taking the infimum over the same set of $w$ proves Eq.~(\plaineqref{eq:alpha_sandwich}).

\subsection{Preservation of the relative-smoothness constant}
\label{subsec:app_beta_preservation}
By the definition of $\widetilde{\beta}_t$,
\[
H(z)\preceq\widetilde{\beta}_t\mathcal{P}_t^{-1}, \qquad \forall z\in\Omega_t.
\]
The left inequality in Eq.~(\plaineqref{eq:inverse_preconditioner_sandwich}) is equivalent to
\[
\mathcal{P}_t^{-1} \preceq (1+\rho_t)\widehat{\mathcal{P}}_t^{-1}.
\]
Hence
\[
H(z) \preceq (1+\rho_t)\widetilde{\beta}_t \widehat{\mathcal{P}}_t^{-1}, \qquad \forall z\in\Omega_t,
\]
so the minimality of $\widehat{\beta}_t$ gives
\[
\widehat{\beta}_t \le (1+\rho_t)\widetilde{\beta}_t.
\]
Conversely, by the definition of $\widehat{\beta}_t$,
\[
H(z)\preceq\widehat{\beta}_t\widehat{\mathcal{P}}_t^{-1}.
\]
The right inequality in Eq.~(\plaineqref{eq:inverse_preconditioner_sandwich}) gives
\[
\widehat{\mathcal{P}}_t^{-1} \preceq \frac{1}{1-\rho_t}\mathcal{P}_t^{-1}.
\]
Therefore
\[
H(z) \preceq \frac{\widehat{\beta}_t}{1-\rho_t}\mathcal{P}_t^{-1},
\]
and the minimality of $\widetilde{\beta}_t$ yields
\[
\widetilde{\beta}_t \le \frac{\widehat{\beta}_t}{1-\rho_t}.
\]
Equivalently,
\[
(1-\rho_t)\widetilde{\beta}_t \le \widehat{\beta}_t.
\]
Combining the two directions proves Eq.~(\plaineqref{eq:beta_sandwich}).

\subsection{Proof of Theorem~\ref{thm:sns_geometry_preservation}}
\label{subsec:app_geometry_preservation}
Equation~(\plaineqref{eq:alpha_sandwich}) follows from Appendix~\ref{subsec:app_alpha_preservation}, and Eq.~(\plaineqref{eq:beta_sandwich}) follows from Appendix~\ref{subsec:app_beta_preservation}. Since all constants are positive,
\[
\frac{\widehat{\alpha}_t}{\widehat{\beta}_t} \ge \frac{(1-\rho_t)\widetilde{\alpha}_t} {(1+\rho_t)\widetilde{\beta}_t} = \frac{1-\rho_t}{1+\rho_t} \frac{\widetilde{\alpha}_t}{\widetilde{\beta}_t},
\]
and similarly
\[
\frac{\widehat{\alpha}_t}{\widehat{\beta}_t} \le \frac{(1+\rho_t)\widetilde{\alpha}_t} {(1-\rho_t)\widetilde{\beta}_t} = \frac{1+\rho_t}{1-\rho_t} \frac{\widetilde{\alpha}_t}{\widetilde{\beta}_t}.
\]
This proves Eq.~(\plaineqref{eq:condition_ratio_sandwich}).

\subsection{Proof of Corollary~\ref{cor:sns_convergence}}
\label{subsec:app_sns_convergence}
\begin{proof}
The assumed segment containment allows us to apply the descent argument of Appendix~\ref{subsec:app_exact_convergence} along the SNS trajectory, using $\widehat{\mathcal{P}}_t$, $\widehat{\alpha}_t$, and $\widehat{\beta}_t$. With $\eta_t=1/\widehat{\beta}_t$,
\[
\mathcal{L}(w_t)-\mathcal{L}^* \le \left(1-\frac{\widehat{\alpha}_t}{\widehat{\beta}_t}\right) \left(\mathcal{L}(w_{t-1})-\mathcal{L}^*\right).
\]
From Theorem~\ref{thm:sns_geometry_preservation}, $\rho_t\le\overline{\rho}$, and the uniform exact bounds,
\begin{align}
\frac{\widehat{\alpha}_t}{\widehat{\beta}_t} \ge \frac{1-\rho_t}{1+\rho_t} \frac{\widetilde{\alpha}_t}{\widetilde{\beta}_t} \ge \frac{1-\overline{\rho}}{1+\overline{\rho}} \frac{\overline{\alpha}}{\overline{\beta}}.
\end{align}
Therefore every contraction factor is bounded above by
\[
1- \frac{1-\overline{\rho}}{1+\overline{\rho}} \frac{\overline{\alpha}}{\overline{\beta}}.
\]
Recursion for $t=1,\ldots,T$ gives Eq.~(\plaineqref{eq:sns_uniform_rate}).
\end{proof}

\section{More Experimental Details}
\label{more_expe}

\subsection{Training and Evaluation Loss Curve}

\begin{figure}[t]
    \centering

    \includegraphics[width=0.49\linewidth]{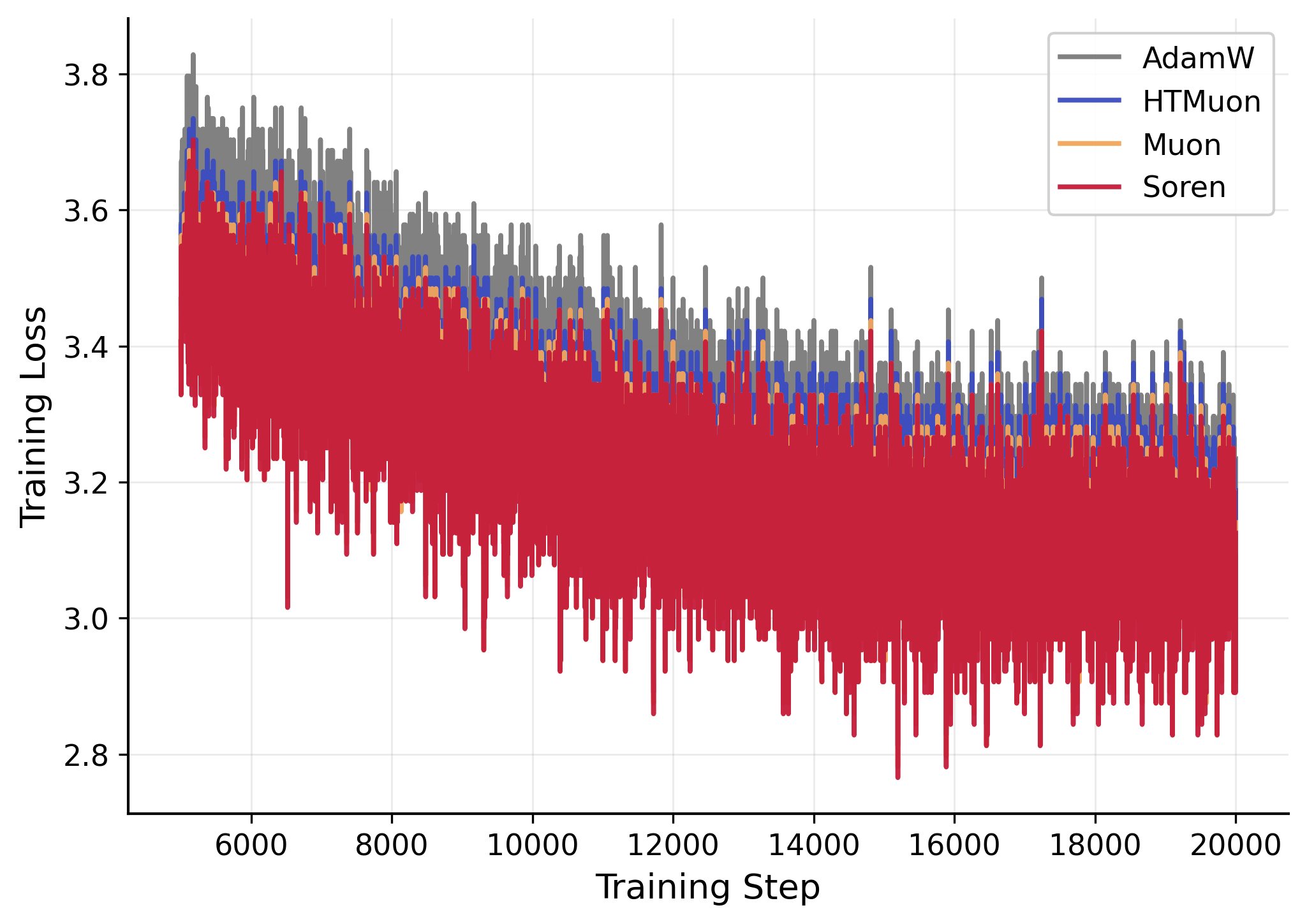}
    \hfill
    \includegraphics[width=0.49\linewidth]{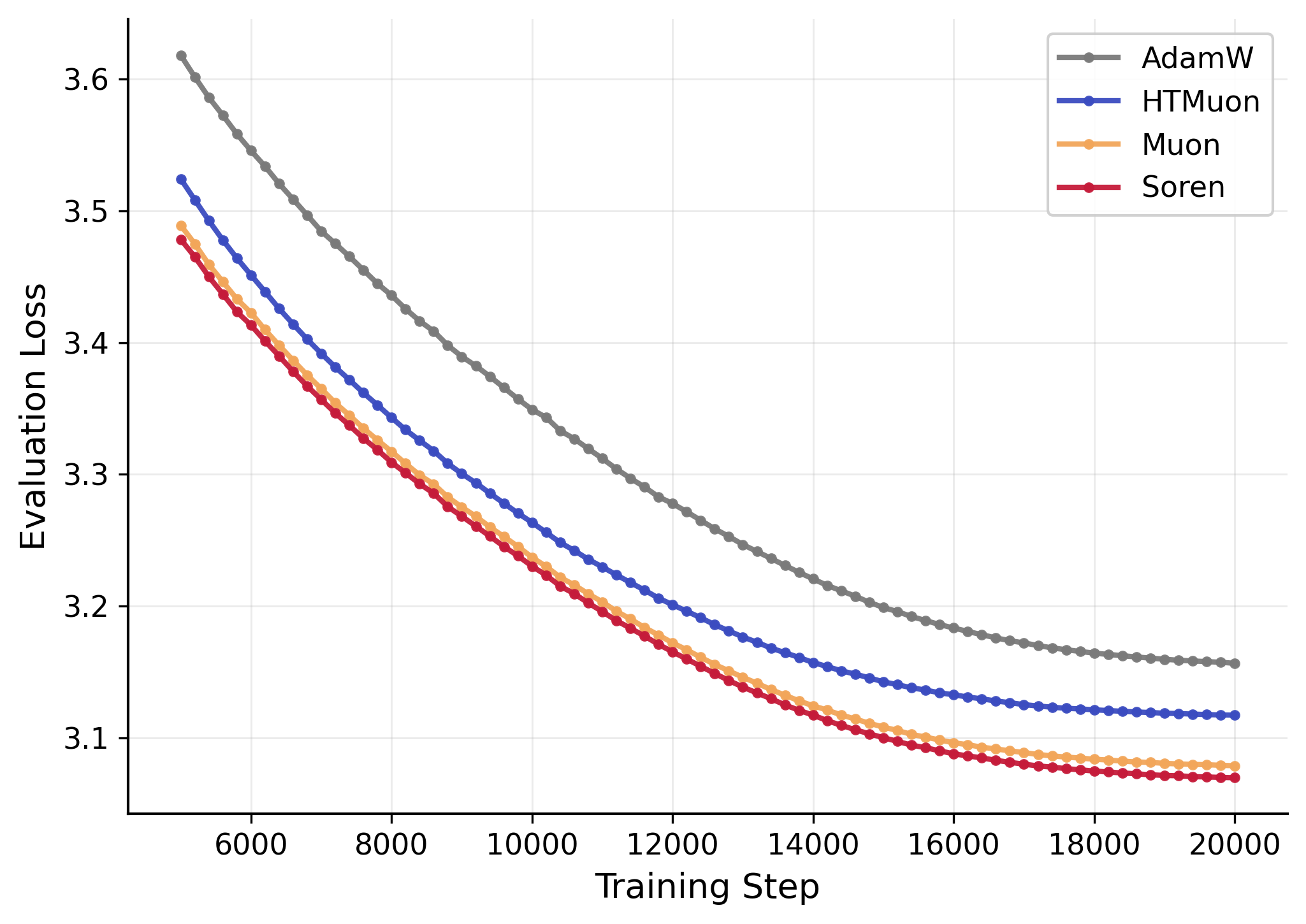}

    \includegraphics[width=0.49\linewidth]{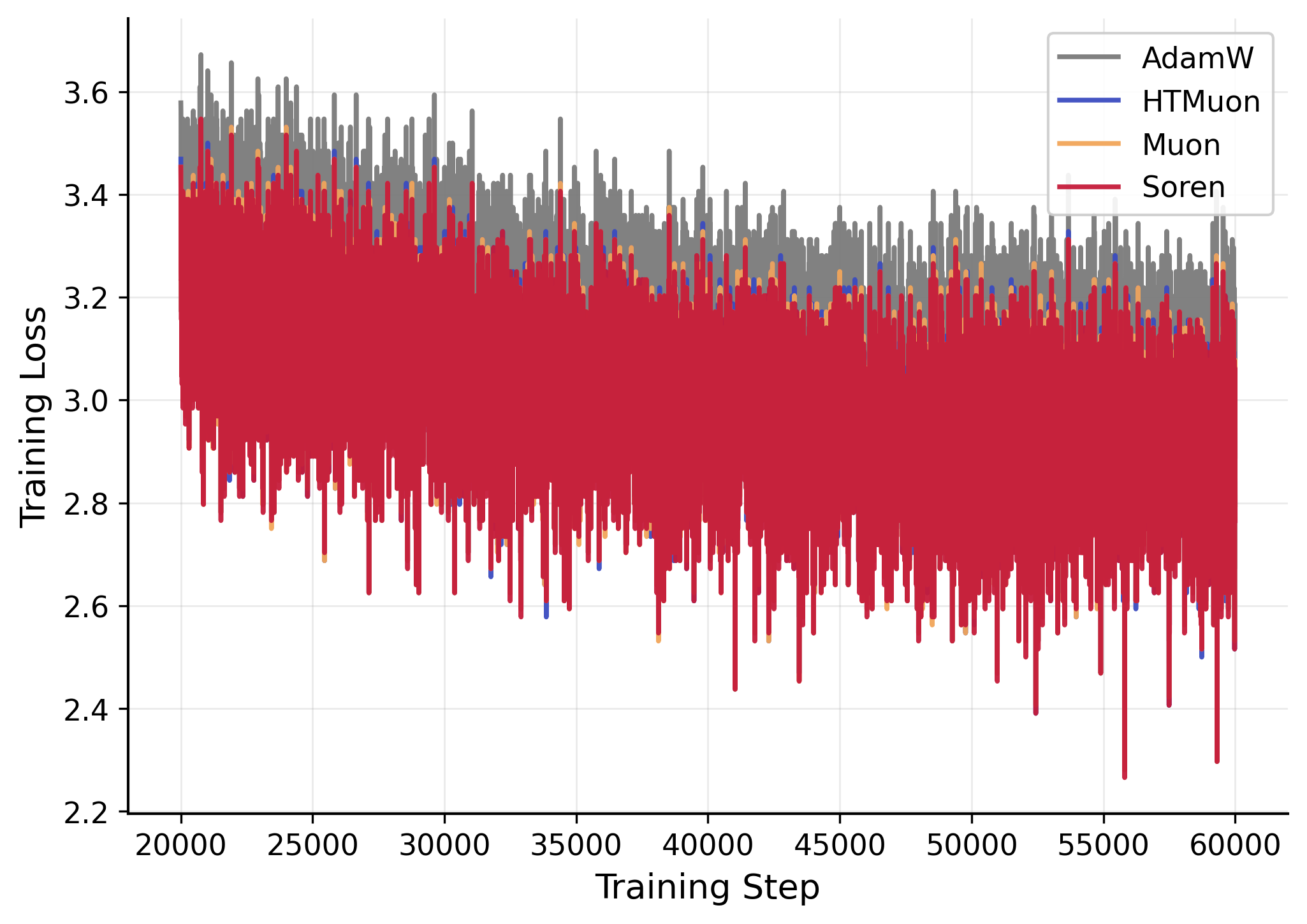}
    \hfill
    \includegraphics[width=0.49\linewidth]{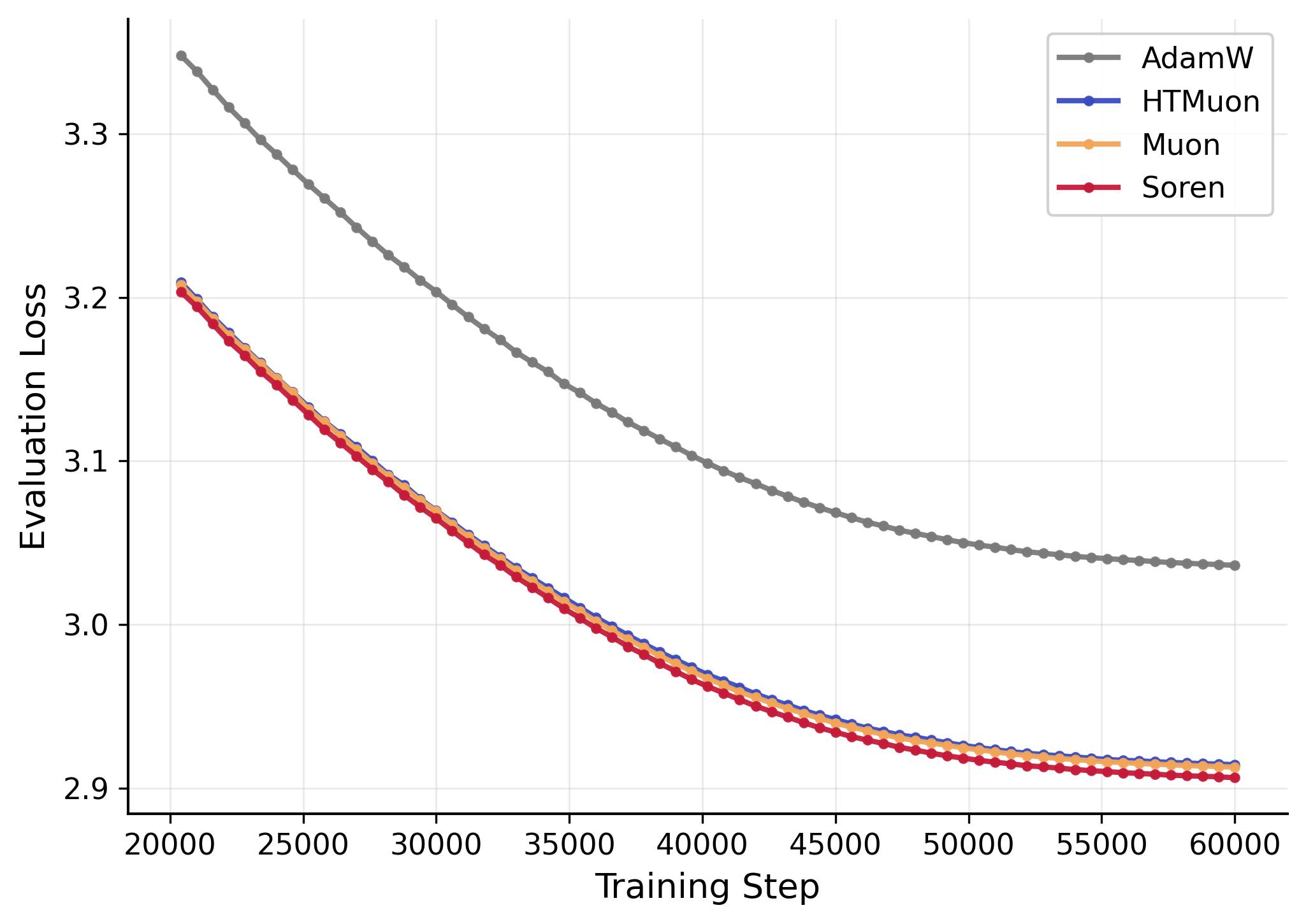}

    \caption{
        Training and evaluation loss curves for the LLaMA-350M (top row) and LLaMA-1B (bottom row) experiments.
    }
    \label{fig:loss1b}
\end{figure}

Figure~\ref{fig:loss1b} presents the training and evaluation loss
curves for LLaMA-350M and LLaMA-1B on the C4 dataset. We observe that Soren
exhibits a convergence trend similar to Muon and HTMuon in terms of training
loss, while achieving a consistently lower evaluation loss in the later stages
of training for both model scales. In particular, the gap between Soren and
the other matrix-based optimizers becomes more apparent toward the end of
training, whereas AdamW maintains a noticeably higher evaluation loss
throughout training. These results indicate that Soren maintains competitive
optimization dynamics while achieving improved generalization performance at
the end of pre-training.

\subsection{More Experiments Results}

\paragraph{Time Complexity}

We measure the wall-clock training time of AdamW, Muon, HTMuon, and Soren to evaluate their practical computational efficiency. We conduct these measurements across pre-training: Llama-60M on C4 using 1 GPU H100, SFT: Qwen2.5-3B on UltraChat 200k using 1 GPU H100, and DPO: Mistral-7B-v0.1 on UltraFeedback Binarized using 2 GPUs H100. For each setting, we adopt the optimal learning rate and hyperparameter configuration reported in Tables~\ref{tab:llama_60m_135m_hyperparameters} and~\ref{tab:sft_dpo_hyperparameters}, respectively.

\begin{table}[htbp]
    \centering
    \caption{Comparison of the wall-clock training time of different optimizers.
    Relative training time is normalized with respect to AdamW.}
    \label{tab:wall_clock_time}

    \begin{tabular}{lcccc}
        \toprule
        & \multicolumn{2}{c}{\textbf{LLaMA-60M}}
        & \textbf{Qwen2.5-3B}
        & \textbf{Mistral-7B-v0.1} \\

        \cmidrule(lr){2-3}

        \textbf{Optimizer}
        & \textbf{Time}
        & \textbf{Relative to AdamW}
        & \textbf{Time}
        & \textbf{Time} \\

        \midrule

        AdamW
        & 0h23m
        & 1.00x
        & 5h23m
        & 10h3m \\

        Muon
        & 0h28m
        & 1.22x
        & 5h33m
        & 10h11m \\

        HTMuon
        & 2h04m
        & 5.39x
        & 5h38m
        & 11h19m \\

        Soren
        & 0h39m
        & 1.69x
        & 5h33m
        & 10h29m \\

        \bottomrule
    \end{tabular}
\end{table}

Table~\ref{tab:wall_clock_time} reports the wall-clock training time of
different optimizers across pre-training, SFT, and DPO settings. Soren
introduces a moderate overhead compared with AdamW and Muon, with a relative
training time of $1.69\times$ that of AdamW on LLaMA-60M. However, this
overhead remains substantially smaller than that of HTMuon, which requires
$5.39\times$ the training time of AdamW in the same setting. For the larger
SFT and DPO experiments, the training time of Soren remains close to that of
AdamW and Muon. Overall, these results highlight a favorable computational
trade-off, where Soren achieves improved performance while maintaining a
moderate wall-clock training cost.

We also provide the FLOPs analysis for Muon and Soren.
Let the initial weight matrix be
$\boldsymbol{W}_0 \in \mathbb{R}^{m \times n}$.

\begin{itemize}
    \item \textbf{FLOPs analysis for Muon (Algorithm~\ref{alg:muon}):}
    \begin{equation}
        F_{\mathrm{Muon}}
        =
        20mnr + O(mn).
    \end{equation}

    \item \textbf{FLOPs analysis for Soren (Algorithm~\ref{alg:soren}):}
    \begin{equation}
        F_{\mathrm{Soren}}
        =
        4mnr(S+2) + O(mn).
    \end{equation}
\end{itemize}

where $r=\min(m,n)$ and $S$ denotes the number of SNS iterations in the
$Q$ stream.

Compared with Muon, Soren incurs additional computational cost primarily
due to the extra Newton--Schulz iterations introduced by the SNS
procedure.

\subsection{Spectral Approximation of the Soren Transform in LLM training}
\label{app:soren_exact_spectra}

\begin{figure}[htbp]
    \centering
    \begin{subfigure}[b]{0.32\textwidth}
        \centering
        \IfFileExists{figures/soren_vs_exact_logcentered_layer7_step000001.png}{%
            \includegraphics[width=\textwidth]{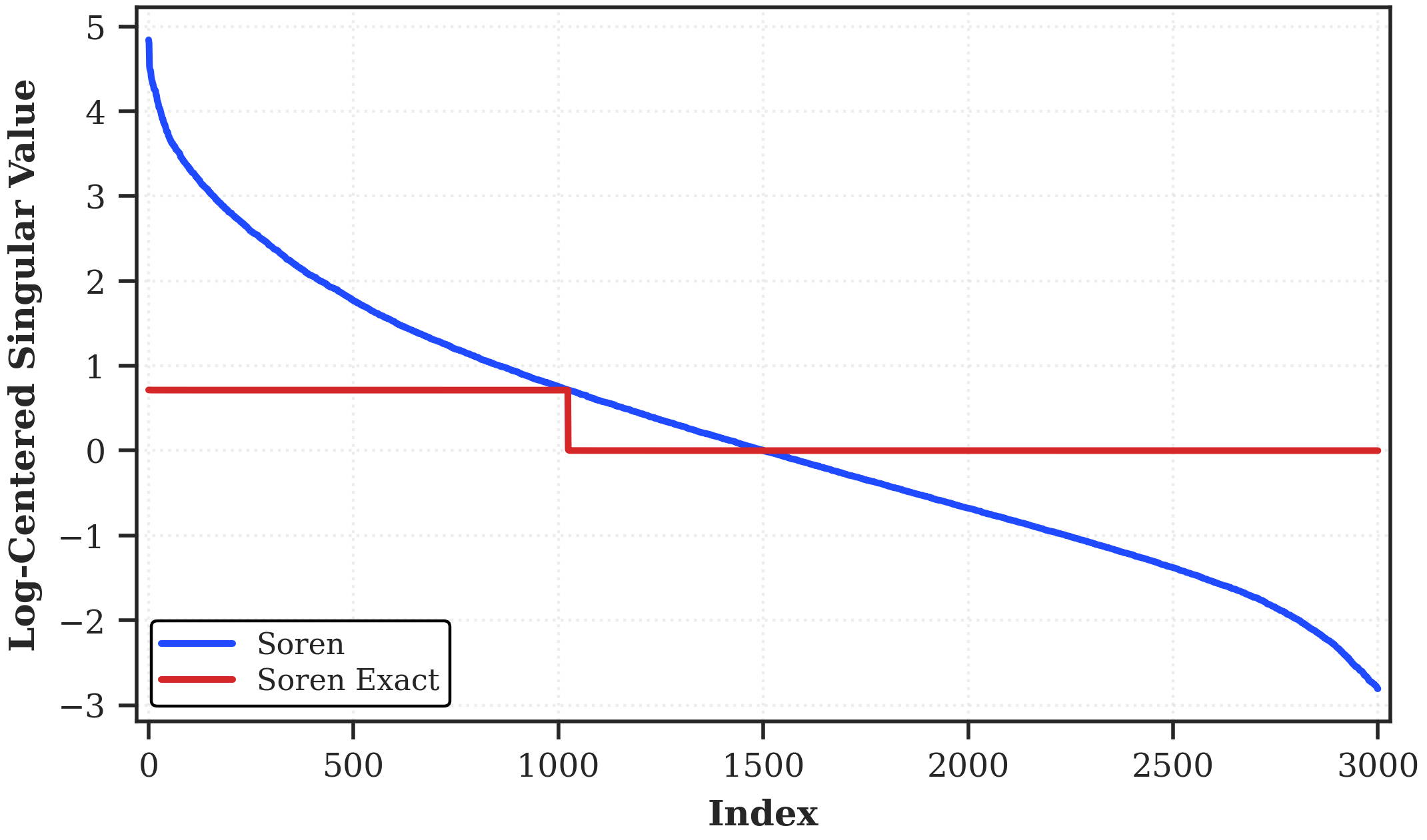}%
        }{\fbox{\parbox[c][3cm][c]{0.9\linewidth}{\centering Missing image: Step 1}}}
        \caption{Step 1}
        \label{fig:soren_exact_step1}
    \end{subfigure}
    \hfill
    \begin{subfigure}[b]{0.32\textwidth}
        \centering
        \IfFileExists{figures/soren_vs_exact_logcentered_layer7_step002500.png}{%
            \includegraphics[width=\textwidth]{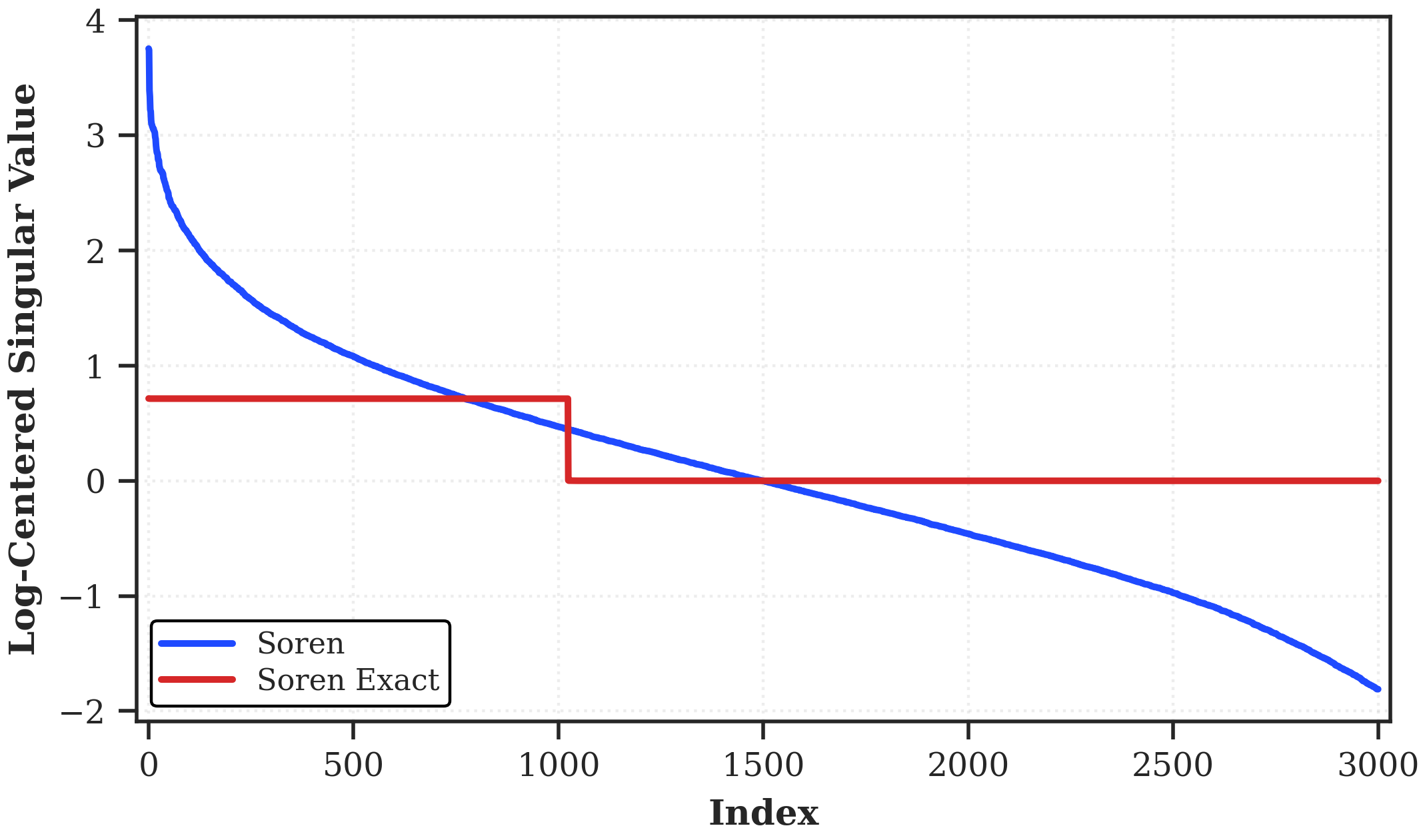}%
        }{\fbox{\parbox[c][3cm][c]{0.9\linewidth}{\centering Missing image: Step 2500}}}
        \caption{Step 2500}
        \label{fig:soren_exact_step2500}
    \end{subfigure}
    \hfill
    \begin{subfigure}[b]{0.32\textwidth}
        \centering
        \IfFileExists{figures/soren_vs_exact_logcentered_layer7_step005000.png}{%
            \includegraphics[width=\textwidth]{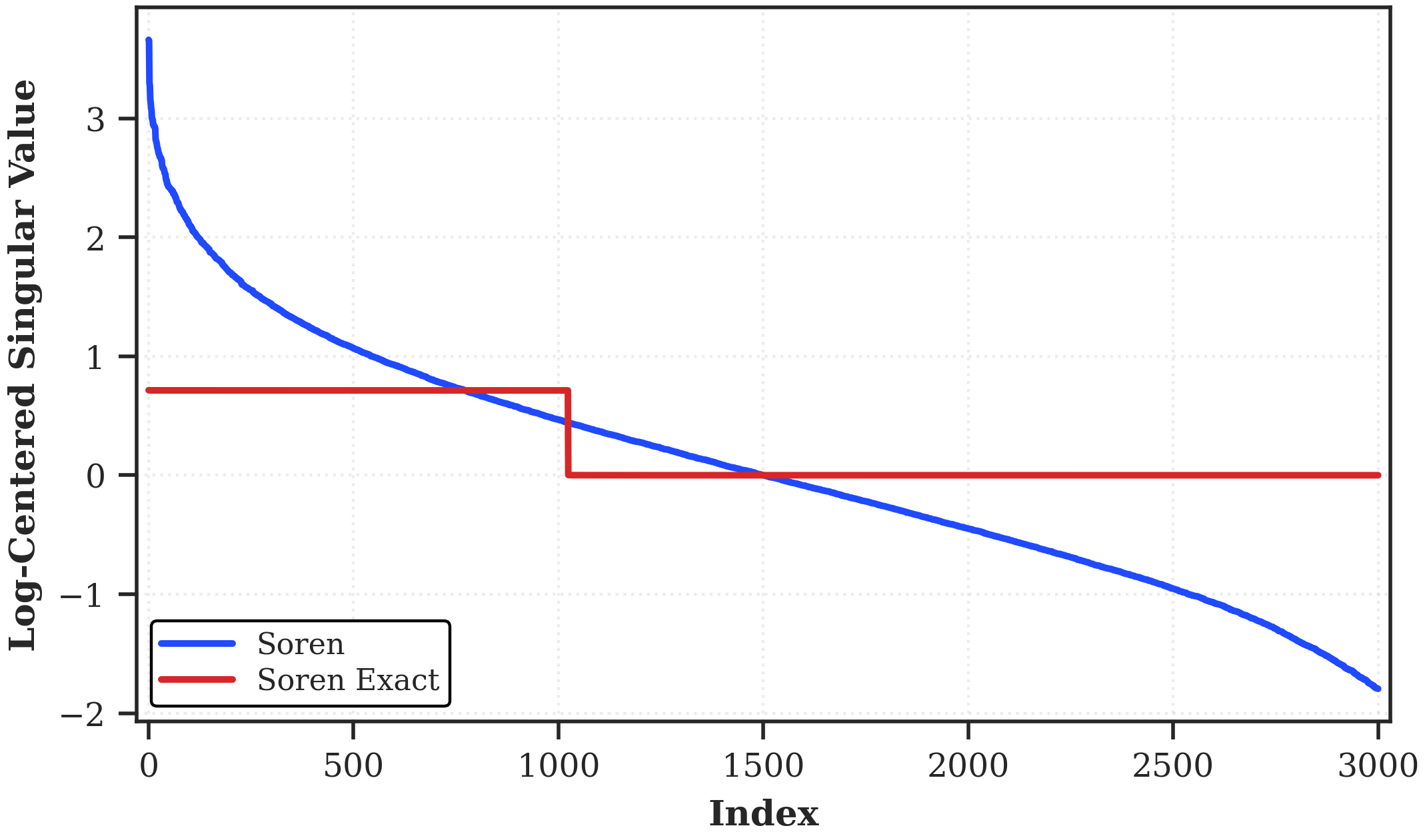}%
        }{\fbox{\parbox[c][3cm][c]{0.9\linewidth}{\centering Missing image: Step 5000}}}
        \caption{Step 5000}
        \label{fig:soren_exact_step5000}
    \end{subfigure}
    \caption{Log-centered singular spectra of the exact sigmoid spectral map and its finite-step Soft Newton--Schulz (SNS) approximation for a representative gradient matrix from the final layer of the Llama-135M model. Log singular values are centered by subtracting the median of the $\log_2$ singular values to isolate relative spectral shape from overall scale. The columns show the initial, intermediate, and later training steps. SNS closely tracks the exact sigmoid transformation across training stages, preserving its relative spectral structure.}
    \label{fig:soren_exact_layer7}
\end{figure}

Figure~\ref{fig:soren_exact_layer7} compares the log-centered singular spectral produced by SNS and the exact sigmoid spectral map at different training stages. The finite-step SNS approximation closely matches the exact sigmoid transformation across all three snapshots, reproducing the overall spectral shape and the relative compression of dominant and weaker modes. This agreement persists throughout training, demonstrating that a small number of Newton--Schulz iterations is sufficient to provide a faithful polynomial approximation to the exact sigmoid spectral map without explicit SVD computation.

\section{Detailed Hyperparameters}
\label{hyperparam}

\subsection{Hyperparameter settings for LLM
pretraining}

Our practical Soren implementation follows the hybrid Muon--AdamW formulation \citep{jordan2024muon}, using standard AdamW hyperparameters for the subset of parameters updated with AdamW. For the C4 experiments, we follow the LLaMA architecture configurations reported in Table~\ref{tab:llama_hyperparameters} to ensure reproducibility. All model variants are trained with a maximum sequence length of 256 and a batch size of 512, corresponding to approximately 13K tokens per batch. We train LLaMA-60M and LLaMA-135M on a single H100 GPU, while LLaMA-350M and LLaMA-1B are trained on four L40S GPUs.

\begin{table}[htbp]
    \centering
    \caption{Hyperparameters of LLaMA models.}
    \label{tab:llama_hyperparameters}
    \begin{tabular}{lccccccc}
        \toprule
        Params & Hidden & Intermediate & Heads & Blocks & Steps & Data amount & Batch Size \\
        \midrule
        60M  & 512  & 1376 & 8  & 8  & 5K & 0.5B & 512 \\
        135M & 768  & 2048 & 12 & 12 & 10K & 1B & 512 \\
        350M & 1024 & 2736 & 16 & 24 & 20K & 2B & 512 \\
        1B   & 2048 & 5461 & 32 & 24 & 60K & 6B & 512 \\
        \bottomrule
    \end{tabular}
\end{table}

For HTMuon, we use $p=0.125$ throughout all LLM pre-training experiments, following the setting identified as optimal in~\citep{htmuon}. The Muon update is applied to both the embedding and output layers during training. Further details on the hyperparameter configurations are provided in Tables~\ref{tab:llama_60m_135m_hyperparameters} and~\ref{tab:llama_350m_1b_hyperparameters}.

\begin{table}[htbp]
    \centering
    \caption{Hyperparameters for LLaMA-60M and LLaMA-135M on the C4 dataset.
    Boldface values denote the optimal hyperparameters.}
    \label{tab:llama_60m_135m_hyperparameters}
    \begin{tabular}{lcc|cc}
        \toprule
        & \multicolumn{2}{c|}{LLaMA-60M}
        & \multicolumn{2}{c}{LLaMA-135M} \\
        \cmidrule(lr){2-3} \cmidrule(lr){4-5}
        Optimizer & LR & WD & LR & WD \\
        \midrule
        AdamW
        & \textbf{1e-3} & 0.1
        & \textbf{1e-3} & 0.1 \\

        Muon
        & \{\textbf{0.01}, 0.02, 0.03, 0.04\} & 0.1
        & \{\textbf{0.01}, 0.02, 0.03, 0.04\} & 0.1 \\

        HTMuon
        & \{0.01, 0.02, \textbf{0.03}, 0.04\} & 0.1
        & \{0.01, 0.02, \textbf{0.03}, 0.04\} & 0.1 \\

        Soren
        & \{0.01, 0.02, \textbf{0.03}, 0.04\} & 0.1
        & \{0.01, 0.02, \textbf{0.03}, 0.04\} & 0.1 \\
        \bottomrule
    \end{tabular}
\end{table}

\begin{table}[htbp]
    \centering
    \caption{Hyperparameters for LLaMA-350M and LLaMA-1B on the C4 dataset.
    Boldface values denote the optimal hyperparameters.}
    \label{tab:llama_350m_1b_hyperparameters}
    \begin{tabular}{lcc|cc}
        \toprule
        & \multicolumn{2}{c|}{LLaMA-350M}
        & \multicolumn{2}{c}{LLaMA-1B} \\
        \cmidrule(lr){2-3} \cmidrule(lr){4-5}
        Optimizer & LR & WD & LR & WD \\
        \midrule

        AdamW
        & \textbf{1e-3} & 0.1
        & \textbf{6e-4} & 0.1 \\

        Muon
        & \{0.0025, \textbf{0.005}, 0.01, 0.015\} & 0.1
        & \{\textbf{0.005}, 0.01\} & 0.1 \\

        HTMuon
        & \{0.0025, \textbf{0.005}, 0.01, 0.015\} & 0.1
        & \{\textbf{0.005}, 0.01\} & 0.1 \\

        Soren
        & \{0.0025, 0.005, \textbf{0.01}, 0.015\} & 0.1
        & \{0.005, \textbf{0.01}\} & 0.1 \\
        
        \bottomrule
    \end{tabular}
\end{table}

\subsection{Hyperparameter Settings for LLM
Supervised Fine-Tuning and LLM Direct Preference Optimization}
\label{app:sft_dpo_hyperparameters}

We report the hyperparameter configurations used for supervised fine-tuning (SFT) and direct preference optimization (DPO) experiments in Table~\ref{tab:sft_dpo_hyperparameters}. For SFT, we use Qwen2.5-3B on the UltraChat 200k dataset, while for DPO, we use Mistral-7B-v0.1 on the UltraFeedback Binarized dataset. For each optimizer, we evaluate a set of learning rates and select the configuration achieving the best validation performance. We use a weight decay of $0.1$ for all experiments. Boldface values denote the selected learning rates.

\begin{table}[htbp]
    \centering
    \caption{Hyperparameters for Qwen2.5-3B and Mistral-7B-v0.1 on the UltraChat 200k and UltraFeedback Binarized datasets. Boldface values denote the selected learning rates.}
    \label{tab:sft_dpo_hyperparameters}
    \begin{tabular}{lcc|cc}
        \toprule
        & \multicolumn{2}{c|}{Qwen2.5-3B}
        & \multicolumn{2}{c}{Mistral-7B-v0.1} \\
        \cmidrule(lr){2-3} \cmidrule(lr){4-5}
        Optimizer & LR & WD & LR & WD \\
        \midrule
        AdamW
        & \{1e-6, 5e-6, 2e-5, \textbf{3e-5}\} & 0.1
        & \{\textbf{1e-7}, 4e-7, 5e-7\} & 0.1 \\

        Muon
        & \{1e-6, 5e-6, \textbf{2e-5}, 3e-5\} & 0.1
        & \{1e-7, 4e-7, \textbf{5e-7}\} & 0.1 \\

        HTMuon
        & \{1e-6, 5e-6, \textbf{2e-5}, 3e-5\} & 0.1
        & \{1e-7, 4e-7, \textbf{5e-7}\} & 0.1 \\

        Soren
        & \{\textbf{1e-6}, 5e-6, 2e-5, 3e-5\} & 0.1
        & \{1e-7, \textbf{4e-7}, 5e-7\} & 0.1 \\
        \bottomrule
    \end{tabular}
\end{table}

\section{More Evaluation Details}
\label{app:sns_metrics}

As a complementary evaluation, we measure the mean absolute error (MAE) and
maximum absolute error (MaxErr) of the SNS approximation on gradients
extracted from a convolutional network~\citep{cnn} trained on
CIFAR-10~\citep{cifar}. For each validation mini-batch, we compute the
output-layer gradient $G$ and normalize it as
$\bar{G}=G/(\|G\|_F+\varepsilon)$. We then obtain the singular value
decomposition
$\bar{G}=U\operatorname{diag}(\sigma_i)V^\top$ and compare the exact sigmoid
coefficients $s_i=\operatorname{sigmoid}(\sigma_i)$ with their SNS
approximations $\widehat{s}_i$. We report the following absolute-error
metrics:
\[
\mathrm{MAE}
=
\frac{1}{r}\sum_{i=1}^{r}
\left|s_i-\widehat{s}_i\right|,
\qquad
\mathrm{MaxErr}
=
\max_{1\leq i\leq r}
\left|s_i-\widehat{s}_i\right|.
\]

\begin{table}[t]
    \centering
    \caption{Approximation error of the SNS iteration relative to the exact
    sigmoid spectral map. Each entry reports the mean $\pm$ standard deviation
    over validation mini-batches.}
    \label{tab:poly_sigmoid_epoch_eval}
    \resizebox{\textwidth}{!}{%
    \begin{tabular}{lccccc}
        \toprule
        \textbf{Metric}
        & \textbf{Epoch 2}
        & \textbf{Epoch 4}
        & \textbf{Epoch 6}
        & \textbf{Epoch 8}
        & \textbf{Epoch 10} \\
        \midrule
        MAE
        & $0.0138_{\pm0.0008}$
        & $0.0139_{\pm0.0005}$
        & $0.0155_{\pm0.0007}$
        & $0.0158_{\pm0.0010}$
        & $0.0163_{\pm0.0010}$ \\
        MaxErr
        & $0.0902_{\pm0.0100}$
        & $0.0893_{\pm0.0069}$
        & $0.0924_{\pm0.0095}$
        & $0.0950_{\pm0.0076}$
        & $0.0971_{\pm0.0112}$ \\
        \bottomrule
    \end{tabular}}
\end{table}

Table~\ref{tab:poly_sigmoid_epoch_eval} summarizes the absolute approximation
errors across training. These results complement the scalar relative-error
analysis in Section~\ref{subsec:sns_iteration_analysis} by evaluating the
approximation on gradient matrices encountered during training. Importantly,
these metrics measure absolute approximation errors and therefore do not
directly quantify the maximum relative error on CNN gradients or establish
the relative-error bound assumed in the convergence analysis.

\section{Optimization Algorithms}
\label{app:algorithms}

In this section, we provide the algorithms for all optimizers evaluated in our study. We use the following notation throughout: $\mathbf{W}_t$ denotes the model parameters at step $t$, $\mathbf{G}_t$ the corresponding gradient, $\eta$ the learning rate, $\lambda$ the weight decay coefficient, $\beta_1$ and $\beta_2$ the first- and second-moment decay rates, respectively, and $\epsilon$ a numerical stability constant. We further use $\|\mathbf{G}_t\|$ to denote the gradient norm, and $\mathbf{m}_t$ and $\mathbf{v}_t$ to denote the first and second moment estimates, respectively. Unless otherwise specified, all operations are element-wise.

\subsection{Soren Optimizer}
\label{app:soren_algorithm}

We summarize the proposed Soren optimizer and its polynomial realization in the
following algorithms. Algorithm~\ref{alg:sns} defines the Soft Newton--Schulz
(SNS) iteration, denoted by $\mathcal{SNS}(\cdot)$, which approximates the
sigmoid spectral map without requiring an explicit SVD. Algorithm~\ref{alg:soren}
presents the complete Soren update, applying $\mathcal{SNS}(\cdot)$ to the
Nesterov-adjusted momentum.

\begin{algorithm}[htbp]
\caption{Soft Newton--Schulz Iteration $\mathcal{SNS}(\boldsymbol{X})$}
\label{alg:sns}
\begin{algorithmic}[1]

\Require Input matrix $\boldsymbol{X} \in \mathbb{R}^{m \times n}$,
steps $K$, numerical stability constant $\epsilon=10^{-8}$

\State $\widehat{\boldsymbol{X}}
\gets
\boldsymbol{X}/(\|\boldsymbol{X}\|_F+\epsilon)$

\State $\boldsymbol{Q}\gets\widehat{\boldsymbol{X}}$

\For{$k=1,\ldots,K$}
    \State $\boldsymbol{Q}
    \gets
    \frac{1}{2}
    (3\boldsymbol{I}-\boldsymbol{Q}\boldsymbol{Q}^{\top})
    \boldsymbol{Q}$
\EndFor

\State $\boldsymbol{T}\gets\frac{1}{4}\boldsymbol{X}$

\For{$j=1,2$}
    \State $\boldsymbol{T}
    \gets
    \frac{1}{2}
    (3\boldsymbol{I}-\boldsymbol{T}\boldsymbol{T}^{\top})
    \boldsymbol{T}$
\EndFor

\State \Return
$\frac{1}{2}\boldsymbol{Q}+\frac{1}{2}\boldsymbol{T}$

\end{algorithmic}
\end{algorithm}

\begin{algorithm}[htbp]
\caption{Soren}
\label{alg:soren}
\begin{algorithmic}[1]

\Require Learning rate $\eta$, momentum coefficient $\mu$

\State Initialize $\boldsymbol{W}_0$ and
$\boldsymbol{M}_0\gets\boldsymbol{0}$

\For{$t=1,\ldots$}

    \State $\boldsymbol{G}_t
    \gets
    \nabla_{\boldsymbol{W}}
    \mathcal{L}(\boldsymbol{W}_{t-1})$

    \State $\boldsymbol{M}_t
    \gets
    \mu\boldsymbol{M}_{t-1}
    +(1-\mu)\boldsymbol{G}_t$

    \State $\boldsymbol{N}_t
    \gets
    (1-\mu)\boldsymbol{G}_t
    +\mu\boldsymbol{M}_t$

    \State $\boldsymbol{O}_t
    \gets
    \mathcal{SNS}(\boldsymbol{N}_t)$

    \State $\boldsymbol{W}_t
    \gets
    \boldsymbol{W}_{t-1}
    -\eta\boldsymbol{O}_t$

\EndFor

\State \Return $\boldsymbol{W}_t$

\end{algorithmic}
\end{algorithm}

\subsection{Baseline Optimizers}
\label{app:baseline_algorithm}

This section presents the algorithmic details of the baseline optimizers
considered in our experiments. We provide the update procedures for HTMuon,
Muon, and AdamW to clarify the optimization
steps used throughout our empirical evaluation. The algorithms are presented
using a unified notation consistent with the formulation of Soren above.

\begin{algorithm}[htbp]
\caption{HTMuon}
\label{alg:htmuon}
\begin{algorithmic}[1]

\Require Initial weights $\boldsymbol{W}_0\in\mathbb{R}^{m\times n}$,
loss function $\mathcal{L}$, learning rate $\eta$,
momentum coefficient $\beta$, weight decay $\lambda$,
power $p\in(0,1)$

\State Initialize $\boldsymbol{M}_0\gets\boldsymbol{0}$

\For{$t=1,2,\ldots$}

    \State $\boldsymbol{G}_t
    \gets
    \nabla_{\boldsymbol{W}}
    \mathcal{L}(\boldsymbol{W}_{t-1})$

    \State $\boldsymbol{M}_t
    \gets
    \beta\boldsymbol{M}_{t-1}
    +(1-\beta)\boldsymbol{G}_t$

    \State $\boldsymbol{U}_t,\boldsymbol{\Sigma}_t,
    \boldsymbol{V}_t^\top
    \gets
    \operatorname{SVD}(\boldsymbol{M}_t)$

    \State $\boldsymbol{O}_t
    \gets
    \boldsymbol{U}_t
    \boldsymbol{\Sigma}_t^p
    \boldsymbol{V}_t^\top$

    \State $s
    \gets
    \sqrt{\max\left(1,\frac{m}{n}\right)}$

    \State $\boldsymbol{W}_t
    \gets
    \boldsymbol{W}_{t-1}
    -\eta\lambda\boldsymbol{W}_{t-1}
    -\eta s\boldsymbol{O}_t$

\EndFor

\State \Return $\boldsymbol{W}_t$

\end{algorithmic}
\end{algorithm}

\begin{algorithm}[htbp]
\caption{Muon}
\label{alg:muon}
\begin{algorithmic}[1]

\Require Learning rate $\eta$, momentum coefficient $\mu$

\State Initialize $\boldsymbol{M}_0\gets\boldsymbol{0}$

\For{$t=1,\ldots$}

    \State $\boldsymbol{G}_t
    \gets
    \nabla_{\boldsymbol{W}}
    \mathcal{L}(\boldsymbol{W}_{t-1})$

    \State $\boldsymbol{M}_t
    \gets
    \mu\boldsymbol{M}_{t-1}
    +\boldsymbol{G}_t$

    \State $\boldsymbol{O}_t
    \gets
    \operatorname{NewtonSchulz}(\boldsymbol{M}_t)$

    \State $\boldsymbol{W}_t
    \gets
    \boldsymbol{W}_{t-1}
    -\eta\boldsymbol{O}_t$

\EndFor

\State \Return $\boldsymbol{W}_t$

\end{algorithmic}
\end{algorithm}

\begin{algorithm}[htbp]
\caption{AdamW}
\label{alg:adamw}
\begin{algorithmic}[1]

\Require Learning rate $\eta$, weight decay $\lambda$,
moment coefficients $\beta_1,\beta_2$,
numerical stability constant $\epsilon$

\State Initialize $\boldsymbol{W}_0$,
$\boldsymbol{m}_0\gets\boldsymbol{0}$,
$\boldsymbol{v}_0\gets\boldsymbol{0}$

\For{$t=1,\ldots$}

    \State $\boldsymbol{G}_t
    \gets
    \nabla_{\boldsymbol{W}}
    \mathcal{L}(\boldsymbol{W}_{t-1})$

    \State $\boldsymbol{m}_t
    \gets
    \beta_1\boldsymbol{m}_{t-1}
    +(1-\beta_1)\boldsymbol{G}_t$

    \State $\boldsymbol{v}_t
    \gets
    \beta_2\boldsymbol{v}_{t-1}
    +(1-\beta_2)\boldsymbol{G}_t^2$

    \State $\widehat{\boldsymbol{m}}_t
    \gets
    \boldsymbol{m}_t/(1-\beta_1^t)$

    \State $\widehat{\boldsymbol{v}}_t
    \gets
    \boldsymbol{v}_t/(1-\beta_2^t)$

    \State $\boldsymbol{O}_t
    \gets
    \widehat{\boldsymbol{m}}_t
    /
    \left(
    \sqrt{\widehat{\boldsymbol{v}}_t}+\epsilon
    \right)$

    \State $\boldsymbol{W}_t
    \gets
    \boldsymbol{W}_{t-1}
    -\eta\boldsymbol{O}_t
    -\eta\lambda\boldsymbol{W}_{t-1}$

\EndFor

\State \Return $\boldsymbol{W}_t$

\end{algorithmic}
\end{algorithm}

\end{document}

%% file: math_commands.tex
\usepackage{amsmath,amsfonts,bm}

\def\eqref#1{equation~\ref{#1}}

\def\plaineqref#1{\ref{#1}}

\def\1{\bm{1}}

\DeclareMathAlphabet{\mathsfit}{\encodingdefault}{\sfdefault}{m}{sl}
\SetMathAlphabet{\mathsfit}{bold}{\encodingdefault}{\sfdefault}{bx}{n}

